\documentclass[10pt]{article}
\IfFileExists{neurips_2026.sty}{%
  \usepackage[preprint]{neurips_2026}
}{%
  \usepackage[letterpaper,margin=1in]{geometry}
  \usepackage[round,authoryear]{natbib}
}
\usepackage[utf8]{inputenc}
\usepackage[T1]{fontenc}
\usepackage{lmodern}
\usepackage{amsmath,amssymb,amsthm,mathtools}
\usepackage{booktabs,array,tabularx,longtable}
\usepackage{algorithm,algpseudocode}
\usepackage{microtype}
\usepackage{xcolor}
\usepackage{hyperref}
\hypersetup{hidelinks,pdftitle={DP-Muon: Differentially Private Optimization via Matrix-Orthogonalized Momentum}}
\allowdisplaybreaks[2]
\newtheorem{theorem}{Theorem}[section]
\newtheorem{lemma}[theorem]{Lemma}
\newtheorem{corollary}[theorem]{Corollary}
\newtheorem{proposition}[theorem]{Proposition}
\theoremstyle{definition}

\newtheorem{assumption}[theorem]{Assumption}
\theoremstyle{remark}

\newcommand{\E}{\mathbb E}
\newcommand{\R}{\mathbb R}
\newcommand{\F}{\mathcal F}
\newcommand{\G}{\mathcal G}
\newcommand{\op}{\mathrm{op}}
\newcommand{\Polar}{\operatorname{Polar}}

\newcommand{\norm}[1]{\left\lVert #1\right\rVert}
\newcommand{\inner}[2]{\left\langle #1,#2\right\rangle}
\newcommand{\pmax}{\mathrm{max,op}}
\newcommand{\psum}{\Sigma,*}
\newcommand{\pos}[1]{\left[#1\right]_+}

\newcommand{\Heat}{\mathsf P}
\newcommand{\Hhat}{\widehat{\mathsf P}}
\newcommand{\bias}{\mathcal B}
\newcommand{\qual}{\mathcal E}
\newcommand{\num}{\mathcal N}

\newcommand{\eps}{\varepsilon}
\title{DP-Muon: Differentially Private Optimization via Matrix-Orthogonalized Momentum}

\author{
Jihwan Kim\thanks{Seoul National University.}
\quad
Chenglin Fan\thanks{ Corresponding author.}
}
\date{}
\begin{document}
\maketitle

\begin{abstract}
We study differentially private optimization with matrix-orthogonalized momentum.
DP-Muon uses conventional global per-example clipping and one Gaussian gradient
release per step; matrix updates and auxiliary updates are post-processing.
Our main contribution concerns the additional mean distortion created when fresh
Gaussian noise passes through a nonlinear matrix map. Conditioning on the actual
adaptive history immediately before the current noise yields an exact Gaussian
heat identity. For a smooth Newton--Schulz map, first-order DP-MuonBC reduces
this conditional output bias from second to fourth order in the fresh noise
scale, and an arbitrary-order extension has bias of order $2K+2$.
We prove matrix-block stationarity bounds under global clipping, retain finite-step orthogonalization
error explicitly, and give an exact criterion for improvement of the resulting
upper bound. A separate inequality exposes the effect of auxiliary Adam updates.
GPT-2 experiments on E2E at four privacy targets favor the
reported Muon configurations over Adam baselines in test NLL.
\end{abstract}

\section{Introduction}
\label{sec:introduction}
Differential privacy (DP) formalizes stability of an algorithm's output to the
addition or removal of a training record \citep{dwork2014algorithmic}.
DP-SGD implements this principle by bounding each example's gradient norm,
perturbing an aggregate with Gaussian noise, and accounting for repeated
accesses to the data \citep{abadi2016deep,mironov2017renyi}.
The same private gradient can support many optimizer transformations.
For matrix-valued parameters, Muon transforms a momentum matrix using a short
Newton--Schulz iteration rather than applying a coordinatewise update
\citep{jordan2024muon}. The question we study is not whether this
post-processing is private, which follows from standard DP, but how its
nonlinearity interacts with privacy noise.

There are two distinct biases. Global clipping changes the mean gradient and
is the sensitivity-control step. Separately, a centered perturbation $Z$
generally satisfies $\E Q(A+Z)\ne Q(A)$ for a nonlinear matrix map $Q$.
Our correction targets only the second effect, with the clipped pre-noise
state as its reference. It does not reconstruct information discarded by
clipping, make the optimization path equal to a nonprivate path, or remove the
variance of the underlying Gaussian release.

The distinction parallels DP-AdamBC, which subtracts the known privacy-noise
contribution from Adam's second-moment estimate \citep{DPAdamBC2024}.
For a scalar quadratic, variance subtraction exactly corrects the noisy
quadratic statistic. It does not make the entire nonlinear Adam update
unbiased. The analogous matrix map in Muon is not quadratic. We develop
inverse-heat extrapolation for this map and explicitly quantify both the
remaining output bias and the correction's update-moment cost.

\paragraph{Contributions and scope.}
We use one conventional global clipping mechanism for every optimizer and
analyze the actual Poisson-sampled empirical-gradient oracle.
Our adaptive conditional-bias theorem does not freeze earlier clipped
gradients. We extend it to arbitrary correction order and to arbitrary fixed
strength, and identify the error caused by a mismatched probe scale.
A finite-probe analysis yields a near-unit \emph{expected} update magnitude
when the normalized noise is small. We then prove a simultaneous matrix-block
stationarity theorem that includes a weighted Newton--Schulz alignment defect,
rather than assuming a random worst-case quality factor is a deterministic
constant. We give a gap-free bound on this defect and an algebraic criterion
for comparing the resulting DP-Muon and DP-MuonBC certificates.
For the hybrid Muon--Adam implementation, we separately account for auxiliary
movement and auxiliary descent; we do not claim an unconditional convergence
theorem for Adam or universal optimizer dominance.

\paragraph{Related work.}
Our privacy accounting uses R\'enyi DP and sampled-Gaussian composition
\citep{mironov2017renyi,mironov2019sampled,wang2019subsampled}.
Adaptive clipping \citep{pichapati2019adaclip} and low-rank
reparametrization \citep{yu2021large} address different aspects of private
optimization. Private language-model fine-tuning provides a natural
application \citep{li2021large}.
Nonprivate Muon analyses study matrix structure and finite-step
orthogonalization \citep{shen2025,kim2026convergence}; these do not by
themselves establish that useful structure survives arbitrary privacy noise.
Section~\ref{sec:structure} gives a conditional perturbation interpretation,
while our principal technical results concern nonlinear output bias and its
optimization cost.

\section{Globally clipped DP-Muon}
\label{sec:algorithm}
\subsection{Problem and private gradient mechanism}
Write the parameters as $\theta=(W,u)$, where
$W=(W^{(1)},\ldots,W^{(L_m)})$ contains the Muon-managed matrices and $u$
contains the remaining parameters. For a fixed dataset
$\mathcal D=\{z_i\}_{i=1}^N$, the empirical objective is
\begin{equation}
 f_{\mathcal D}(\theta)=\frac1N\sum_{i=1}^N\ell(\theta;z_i).
 \label{eq:objective}
\end{equation}
A randomized algorithm $\mathcal A$ is $(\eps,\delta)$-DP if, for every
add/remove adjacent pair $\mathcal D,\mathcal D'$ and measurable set $S$,
\begin{equation}
 \Pr\{\mathcal A(\mathcal D)\in S\}
 \le e^\eps\Pr\{\mathcal A(\mathcal D')\in S\}+\delta.
 \label{eq:dp-definition}
\end{equation}
Privacy is record-level, not automatically user-level or prompt-group-level.

Let $\pi\in(0,1]$ be a public inclusion probability. At step $t$, independently
include each record with probability $\pi$, obtaining a Poisson lot
$\mathcal L_t$. Use a \emph{deterministic public} normalization $B>0$;
for the analyzed dataset, $B=\pi N$. Both $\pi$ and $B$ are kept unchanged
when comparing adjacent datasets. They are not recomputed from a changed
private dataset size, and $B$ is not the realized lot size.
Let $G_{t,i}=\nabla_\theta\ell(\theta_{t-1};z_i)$ and concatenate all
trainable coordinates when computing its Euclidean/Frobenius norm
$\norm{G_{t,i}}_2$. With public $C>0$ and numerical stabilizer
$\tau_c\ge0$, define
\begin{align}
 a_{t,i}&=\min\left\{1,\frac{C}{\norm{G_{t,i}}_2+\tau_c}\right\},
 &H_{t,i}&=a_{t,i}G_{t,i}, \label{eq:global-clip}\\
 \bar G_t&=\frac1B\sum_{i\in\mathcal L_t}H_{t,i},
 &\widetilde G_t&=\bar G_t+Z_t,
 \quad Z_t\sim\mathcal N(0,\nu^2 I),
 \quad\nu=\frac{\sigma C}{B}. \label{eq:global-release}
\end{align}
At a zero gradient set $a_{t,i}=1$ if the displayed ratio is undefined.
The same scalar $a_{t,i}$ multiplies \emph{every} coordinate of example $i$.
This is one joint Gaussian release; splitting its coordinates between Muon
and Adam does not create additional gradient queries.
An empty lot still uses the specified Gaussian release and accounting step.

\subsection{Matrix post-processing and the smooth map}
For a matrix of shape $m\times n$, let $\mathcal T$ transpose it when
$m>n$, and otherwise act as the identity. For $\gamma>0$ and integers
$\kappa,q\ge1$, define
\begin{align}
 p_\kappa(z)&=\sum_{j=0}^{\kappa}c_j(1-z)^j,
 &c_j&=\frac{\binom{2j}{j}}{4^j}, \label{eq:ns-poly}\\
 Y_0(A)&=\frac{\mathcal T(A)}{\sqrt{\gamma^2+\norm{A}_F^2}},
 &Y_{j+1}(A)&=p_\kappa\big(Y_j(A)Y_j(A)^\top\big)Y_j(A),\\
 Q(A)&=Q_{\gamma,\kappa,q}(A)=\mathcal T^{-1}\big(Y_q(A)\big).
 \label{eq:smooth-map}
\end{align}
This operator-norm-bounded \emph{Smooth-NS} map is globally smooth, has bounded
fixed-order derivatives, and satisfies $\norm{Q(A)}_\op\le1$
(Appendix~\ref{app:ns}). Different shapes may use different $\gamma_\ell$.
The practical Jordan polynomial used by some experiments is a different map;
it is not silently identified with \eqref{eq:smooth-map}.

For each matrix block, DP-Muon uses raw momentum
\begin{equation}
 M_t^{(\ell)}=\beta M_{t-1}^{(\ell)}+\widetilde G_t^{(\ell)},
 \qquad M_0^{(\ell)}=0,
 \qquad O_t^{\rm mu,(\ell)}=Q_\ell(M_t^{(\ell)}).
 \label{eq:momentum}
\end{equation}
DP-MuonBC replaces $O_t^{\rm mu,(\ell)}$ by the corrected direction in
Section~\ref{sec:bias}. The privacy theorem permits auxiliary Adam updates,
public learning-rate rescaling, and other post-processing. The matrix
stationarity theorem uses a common stepsize and no weight decay; its extension
to moving auxiliary coordinates is stated separately.

\begin{algorithm}[t]
\caption{Global-clipping DP-Muon and first-order DP-MuonBC}
\label{alg:global}
\begin{algorithmic}[1]
\Require Public $T,B,\pi,C,\tau_c,\sigma,\beta$, learning rates, maps $Q_\ell$, probe counts $J_\ell$
\Require Fixed strength $\lambda\ge0$; $\lambda=0$ gives Smooth-NS DP-Muon, $\lambda=1$ gives DP-MuonBC
\State Initialize public parameters and $M_0^{(\ell)}=0$; set $\nu=\sigma C/B$
\For{$t=1,\ldots,T$}
 \State Draw a Poisson lot $\mathcal L_t$ with inclusion probability $\pi$
 \State Compute the full per-example gradients and one global multiplier $a_{t,i}$ using \eqref{eq:global-clip}
 \State Form the single private release $\widetilde G_t$ using \eqref{eq:global-release}
 \For{each Muon matrix $\ell$}
  \State $M_t^{(\ell)}\gets\beta M_{t-1}^{(\ell)}+\widetilde G_t^{(\ell)}$
  \State $O_t^{(\ell)}\gets Q_\ell(M_t^{(\ell)})$
  \If{$\lambda>0$}
   \State Draw $J_\ell$ independent standard Gaussian matrices $U_j$, independent of training noise
   \State $\widehat H\gets(2J_\ell)^{-1}\sum_j[Q_\ell(M_t^{(\ell)}+\nu U_j)+Q_\ell(M_t^{(\ell)}-\nu U_j)]$
   \State $O_t^{(\ell)}\gets(1+\lambda)Q_\ell(M_t^{(\ell)})-\lambda\widehat H$
  \EndIf
 \EndFor
 \State Simultaneously update $W_t^{(\ell)}\gets W_{t-1}^{(\ell)}-\eta_t O_t^{(\ell)}$
 \State If auxiliary parameters are present, update them using only $\widetilde G_t^u$ and private optimizer state
\EndFor
\end{algorithmic}
\end{algorithm}

\subsection{Privacy and the meaning of a matched budget}
\begin{theorem}[Global-release privacy]
\label{thm:dp-muon-privacy}
Suppose all training-data access is through \eqref{eq:global-release}, with
public hyperparameters and independent Gaussian noise at each step.
Let $\rho_{\rm SGM}(\alpha;\pi,\sigma)$ be a valid order-$\alpha$ R\'enyi
DP bound for the Poisson-sampled Gaussian mechanism under add/remove
adjacency. Then Algorithm~\ref{alg:global}, including its auxiliary updates,
is $(\eps,\delta)$-DP for
\begin{equation}
 \eps=\inf_{\alpha\in\mathcal A}
 \left\{T\rho_{\rm SGM}(\alpha;\pi,\sigma)
       +\frac{\log(1/\delta)}{\alpha-1}\right\},
 \qquad \mathcal A\subset(1,\infty).
 \label{eq:rdp-accountant}
\end{equation}
A valid tighter RDP-to-DP conversion may also be used.
The certificate is unchanged by the number of probe evaluations or the
choice of matrix post-processing.
\end{theorem}
The globally clipped contribution has norm at most $C$, so the base averaged
query has add/remove sensitivity at most $C/B$. Standard sampled-Gaussian
accounting, adaptive composition, and post-processing prove the theorem
\citep{mironov2017renyi,mironov2019sampled}; details are in
Appendix~\ref{app:privacy}. The guarantee covers the full model, not only
its matrix blocks. It applies to the mathematical Gaussian mechanism;
a deployed implementation must also implement its sampling and noise correctly.

Two optimizer trajectories generally query different gradients. ``Same
privacy'' means the same sampling, sensitivity, noise calibration and
composition certificate, not an identical realized transcript. The reported
experimental budgets are per configured training run. A complete private
hyperparameter search requires composition or a separately justified private
selection procedure; selecting among several private runs is not free merely
because the validation set is public.

\section{Adaptive inverse-heat bias correction}
\label{sec:bias}
\subsection{The conditional target and exact bias identity}
Suppress the matrix-block superscript. Let $\F_{t-1}$ contain the complete
training history, including previous minibatches, privacy noises, probes and
auxiliary updates. Define
\begin{equation}
 \G_t=\sigma(\F_{t-1},\mathcal L_t),\qquad
 A_t=\beta M_{t-1}+\bar G_t.
 \label{eq:pre-noise-state}
\end{equation}
With deterministic per-example losses, the current clipped lot is
$\G_t$-measurable. Any independent randomness used in current gradient
evaluation is also included before drawing $Z_t$. The actual recursion gives
\begin{equation}
 M_t=A_t+\nu E_t,\qquad
 E_t\mid\G_t\sim\mathcal N(0,I_d),\qquad d=mn.
 \label{eq:adaptive-gaussian}
\end{equation}
Only the \emph{current} Gaussian noise is independent of this state. No
independence between earlier gradients and earlier privacy noises is required.
The target $Q(A_t)$ is the map applied to the realized adaptive pre-noise
momentum, not the direction of a noise-free training run.

For a matrix-valued function $F$, set
\begin{equation}
 (\Heat_\rho F)(X)=\E_U F(X+\rho U),\qquad
 \Heat_{\rho_1}\Heat_{\rho_2}=\Heat_{\sqrt{\rho_1^2+\rho_2^2}}.
 \label{eq:heat}
\end{equation}
Independent antithetic probes estimate it by
\begin{equation}
 \Hhat_{\rho,J}Q(X)=\frac1{2J}\sum_{j=1}^J
    [Q(X+\rho U_j)+Q(X-\rho U_j)].
 \label{eq:heat-estimator}
\end{equation}
At step $t$, the probe matrices are independent standard Gaussians,
independent of $(\mathcal G_t,E_t)$. For matched fresh-noise scale
$\rho=\nu$ and fixed $\lambda\ge0$, define
\begin{equation}
 O_t^{\lambda,J}=(1+\lambda)Q(M_t)-\lambda\Hhat_{\nu,J}Q(M_t).
 \label{eq:bc-direction}
\end{equation}
Let $\Delta$ be the entrywise Laplacian, applied componentwise to $Q$, and
write
\begin{equation}
 M_{2j}=\sup_X\norm{\Delta^jQ(X)}_\op<\infty,\qquad
 a=\nu^2/2,\qquad D_\nu=\Heat_\nu-I.
 \label{eq:derivative-constants}
\end{equation}

\begin{theorem}[Fully adaptive bias, including correction strength]
\label{thm:adaptive-bias}
Under \eqref{eq:adaptive-gaussian}, with fresh independent probes,
\begin{align}
 \E[O_t^{\lambda,J}\mid\G_t]-Q(A_t)
 &=\big[(1-\lambda)D_\nu-\lambda D_\nu^2\big]Q(A_t),
 \label{eq:strength-identity}\\
 \norm{\E[O_t^{\lambda,J}\mid\G_t]-Q(A_t)}_\op
 &\le |1-\lambda|aM_2+\lambda a^2M_4.
 \label{eq:strength-bias}
\end{align}
In particular, the bounds are $\nu^2M_2/2$ for ordinary Smooth-NS
DP-Muon and $\nu^4M_4/4$ for DP-MuonBC with $\lambda=1$.
They hold nonasymptotically along the fully adaptive trajectory, for every
$J\ge1$.
\end{theorem}
Conditioning first on $M_t$ and then on $\G_t$ gives
$\E[O_t^{\lambda,J}\mid\G_t]=[(1+\lambda)\Heat_\nu-\lambda\Heat_\nu^2]Q(A_t)$.
Writing $\mathsf S_s=\Heat_{\sqrt{2s}}$, the identities
\begin{equation}
 (\mathsf S_a-I)^jQ
 =\int_{[0,a]^j}\mathsf S_{s_1+\cdots+s_j}\Delta^jQ\,ds_1\cdots ds_j
 \label{eq:heat-remainder}
\end{equation}
and contraction of Gaussian averaging imply the two norm bounds.
Appendix~\ref{app:heat} supplies all regularity and conditioning details.

\subsection{Higher order and why calibration matters}
For $K\ge0$, let $c_{K,j}=(-1)^j\binom{K+1}{j+1}$ and
\begin{equation}
 O_t^{[K]}=\sum_{j=0}^K c_{K,j}
       \Hhat_{\sqrt j\nu,J_j}Q(M_t),\qquad
 \Hhat_{0,J_0}Q=Q.
 \label{eq:hierarchy}
\end{equation}
The same argument gives
\begin{equation}
 \E[O_t^{[K]}\mid\G_t]=[I-(I-\Heat_\nu)^{K+1}]Q(A_t),\qquad
 \norm{\E[O_t^{[K]}\mid\G_t]-Q(A_t)}_\op
 \le a^{K+1}M_{2K+2}.
 \label{eq:hierarchy-bias}
\end{equation}
The $K=2$ direction is $3Q-3\Hhat_\nu Q+\Hhat_{\sqrt2\nu}Q$.
The deterministic coefficient bound grows to $2^{K+1}-1$, so higher order
does not imply a better optimization guarantee. Probe families may be shared
across heat levels; each estimator must retain the stated marginal mean and
be independent of the pre-noise information and the current privacy noise.

\begin{proposition}[Probe-scale mismatch]
\label{prop:mismatch}
If \eqref{eq:bc-direction} instead uses a fixed probe scale $\rho$, then
\begin{equation}
 \E[O_t\mid\G_t]
 =[(1+\lambda)\Heat_\nu-\lambda\Heat_{\sqrt{\nu^2+\rho^2}}]Q(A_t).
 \label{eq:mismatch-exact}
\end{equation}
With $c=\rho^2/2$, its bias equals
$(a-\lambda c)\Delta Q(A_t)+R_t$, where
\begin{equation}
 \norm{R_t}_\op\le
 \left(\frac{a^2}{2}+\lambda ac+\frac{\lambda c^2}{2}\right)M_4.
 \label{eq:mismatch-remainder}
\end{equation}
The assertion also holds conditionally for pre-noise measurable scales and
strengths; privacy additionally requires that such choices use only public
information or already private state.
\end{proposition}
At normalized momentum $M_t/s_t$, with
$s_t=\sum_{j=0}^{t-1}\beta^j$, the fresh noise standard deviation is
$\nu/s_t$. An accumulated scale
$\sqrt{V_t}/s_t$, $V_t=\beta^2V_{t-1}+\nu^2$, is different. It does not
turn the adaptive accumulated gradient path into an independent clean signal.
Equation~\eqref{eq:mismatch-exact} describes its one-step mean, but the matched
fourth-order theorem must not be applied without checking calibration.
A scale or strength chosen as a function of the \emph{current noisy} matrix
requires a separate argument.

\paragraph{What is, and is not, corrected.}
For $F(x)=x^2$, $(2I-\Heat_\nu)F(x)=x^2-\nu^2$ and
$\Heat_\nu(2I-\Heat_\nu)F=F$. This explains the variance-subtraction
analogy with DP-AdamBC \citep{DPAdamBC2024}, not exact unbiasedness of
Adam's division by a second-moment estimate. Clipping bias and previously
accumulated trajectory effects remain in both methods. Even a smaller output
bias does not by itself imply better alignment with the current objective
gradient; Appendix~\ref{app:no-dominance} gives a simple counterexample.

\section{Finite-horizon optimization guarantees}
\label{sec:stationarity}
\subsection{Matrix geometry, global clipping, and finite-step error}
For the main theorem, optimize all $L_m$ matrices simultaneously while holding
auxiliary coordinates fixed. The objective need not separate across matrices.
Write $m_\ell\times n_\ell$ for block shapes,
$r_\ell=\min(m_\ell,n_\ell)$, $d_\ell=m_\ell n_\ell$, and
\begin{equation}
 R=\sum_\ell r_\ell,\quad D=\sum_\ell d_\ell,\quad
 \norm{H}_{\pmax}=\max_\ell\norm{H^{(\ell)}}_\op,\quad
 \norm{G}_{\psum}=\sum_\ell\norm{G^{(\ell)}}_*.
 \label{eq:product-norms}
\end{equation}
The last two norms are dual. Let $f$ denote the restricted empirical objective,
$g_t=\nabla_W f(W_{t-1})$, and $D_0=f(W_0)-f_*<\infty$.

\begin{assumption}[Product-space smoothness]
\label{ass:smoothness}
The objective is bounded below by $f_*$ and, for all $X,Y$,
$\norm{\nabla f(X)-\nabla f(Y)}_{\psum}
\le L\norm{X-Y}_{\pmax}$.
\end{assumption}

Define the exact global clipping residual and a variance envelope by
\begin{align}
 \bias_T(C)&=\sqrt R\sup_{t\le T}\E\left[
       \frac1N\sum_{i=1}^N(1-a_{t,i})\norm{G_{t,i}}_2\right],
 \label{eq:clip-residual}\\
 v_C^2&=\frac{(1-\pi)C^2}{B}+D\nu^2. \label{eq:variance-envelope}
\end{align}
Here the norm in \eqref{eq:clip-residual} is over all coordinates used by global
clipping, including auxiliary coordinates if queried. For exact clipping,
$\tau_c=0$, this residual is
$\sqrt R\sup_t\E N^{-1}\sum_i(\norm{G_{t,i}}_2-C)_+$.
For $\tau_c>0$, that tail expression plus $\sqrt R\tau_c$ is an upper bound.
Thus frequent clipping is allowed and is charged explicitly.

Set
\begin{equation}
 s_t=\sum_{j=0}^{t-1}\beta^j,\quad
 w_{t,j}=\beta^{t-j}/s_t,\quad \bar A_t=A_t/s_t,\quad
 \Gamma_T(\beta)=\frac1T\sum_{t=1}^T\sqrt{\sum_{j=1}^t w_{t,j}^2}.
 \label{eq:weights}
\end{equation}
The Poisson oracle is a martingale-difference oracle after subtracting its
clipped empirical mean. If the update has
$\sup_t\E\norm{O_t}_{\pmax}\le K$, its tracking bound is
\begin{equation}
 \frac1T\sum_{t=1}^T\E\norm{\bar A_t-g_t}_{\psum}
 \le\bias_T(C)+\sqrt R\,v_C\Gamma_T(\beta)
       +\frac{L\eta K\beta}{1-\beta}.
 \label{eq:tracking-main}
\end{equation}
The optimization theorems take the per-example loss to be deterministic given
its parameters and record. Additional gradient-evaluation randomness requires
its own oracle-variance analysis; the conditional heat identity remains valid
when that randomness is included in $\G_t$. No independence of gradients
across time is assumed. The proof also gives
$\Gamma_T\le c_1/(T\sqrt{1-\beta})+c_2\sqrt{1-\beta}$ and a sharper
variance expression omitting the current noise from $A_t$
(Appendix~\ref{app:tracking}).

To retain finite-step approximation error without a uniform spectral gap,
choose a \emph{deterministic} $\alpha\in(0,1]$ and define
\begin{equation}
 \qual_T(\alpha)=\frac1T\sum_{t=1}^T\E\left[
 \frac1{s_t}\sum_{\ell=1}^{L_m}
 \pos{\alpha\norm{A_t^{(\ell)}}_*
       -\inner{A_t^{(\ell)}}{Q_\ell(A_t^{(\ell)})}_F}\right].
 \label{eq:quality-defect}
\end{equation}
At $\alpha=1$, this is the nuclear-norm-weighted polar-alignment loss of
finite-step Smooth-NS. It remains well defined at zero and rank-deficient
matrices. A deterministic bound
$\norm{Q_\ell(A_t^{(\ell)})-\Polar(A_t^{(\ell)})}_\op\le1-\alpha$
would imply $\qual_T(\alpha)=0$, but is not required.
For common $\kappa,q$ and $\gamma_{\max}=\max_\ell\gamma_\ell$, we prove
\begin{equation}
 \qual_T(1)\le
 \frac{R(\gamma_{\max}+C+\nu\sqrt D)}
 {\sqrt{2e(\kappa+1)^q}}.
 \label{eq:gap-free-quality}
\end{equation}
This gap-free estimate is conservative but deterministic. It does not place
a random path maximum outside an expectation. Its scalar proof and exact
singular-value expression appear in Appendix~\ref{app:ns}.

\subsection{Refined finite-probe moments}
For block $\ell$, let the Hessian envelope be
\begin{equation}
 H_{2,\ell}=\sup_X\sup_{\norm{U}_F=\norm{V}_F=1}
       \norm{D^2Q_\ell(X)[U,V]}_F.
 \label{eq:hessian-envelope}
\end{equation}
It is finite for fixed Smooth-NS parameters. Set
\begin{equation}
 \Xi_{\rm all}^2=\sum_\ell
 \left(\frac{H_{2,\ell}\nu^2}{2}\right)^2
        \left(d_\ell^2+\frac{2d_\ell}{J_\ell}\right),\qquad
 K_\lambda=\min\{1+2\lambda,\ 1+\lambda\Xi_{\rm all}\}.
 \label{eq:moment-constant}
\end{equation}

\begin{proposition}[Update moments and finite probes]
\label{prop:moments}
For the simultaneous strength-$\lambda$ updates,
\begin{equation}
 \E\norm{O_t}_{\pmax}\le K_\lambda,
 \qquad \E\norm{O_t}_{\pmax}^2\le K_\lambda^2.
 \label{eq:update-moments}
\end{equation}
For one block and any deterministic center $X$, the probe error
$\zeta=\Hhat_{\nu,J}Q(X)-\Heat_\nu Q(X)$ has mean zero and
\begin{equation}
 \E\norm{\zeta}_F^2
 \le\frac1J\min\left\{r,
         \frac{H_2^2\nu^4}{4}d(d+2)\right\}.
 \label{eq:probe-variance}
\end{equation}
These estimates hold conditionally on the noisy center as well.
\end{proposition}
Antithetic pairing cancels the first derivative, leaving a second difference
bounded by $H_2\nu^2\norm{U}_F^2/2$. Its empirical chi-square second
moment is $d^2+2d/J$, yielding \eqref{eq:moment-constant}.
For $\lambda=1$, the deterministic bound $\norm{O_t}_{\pmax}\le3$ remains
true, but the expected descent penalty can use $K_1^2$ rather than $9$.
The improvement is $1+O(\nu^2)$ only for fixed shapes, maps and regularizers.
In particular, $H_2=O(\gamma^{-2})$ and
$M_{2j}=O(d^j\gamma^{-2j})$, with map-dependent constants.
The relevant small-noise quantity is approximately $\sqrt d\,\nu/\gamma$,
not the coordinate noise alone. Increasing $J$ removes Monte Carlo variance;
it does not remove the correction's nonzero mean or the original DP noise.

\subsection{Stationarity and comparison of certificates}
For $K\ge1$ define
\begin{equation}
 \num_T(K;\alpha)=\frac{D_0}{\eta T}
 +(\alpha+1)\left[\bias_T(C)+\sqrt R\,v_C\Gamma_T(\beta)
                     +\frac{L\eta K\beta}{1-\beta}\right]
 +\qual_T(\alpha)+\frac{L\eta K^2}{2}.
 \label{eq:certificate-numerator}
\end{equation}

\begin{theorem}[Simultaneous matrix-block stationarity]
\label{thm:stationarity}
Assume the Poisson mechanism above, Assumption~\ref{ass:smoothness}, and
constant stepsize $\eta>0$ and momentum $\beta\in[0,1)$ for the chosen
horizon. Let $b\ge0$ and $K\ge1$ be deterministic constants. Suppose, for all $t$ and $\ell$,
\begin{equation}
 \norm{\E[O_t^{(\ell)}\mid\G_t]-Q_\ell(A_t^{(\ell)})}_\op\le b,
 \quad \E\norm{O_t}_{\pmax}\le K,
 \quad \E\norm{O_t}_{\pmax}^2\le K^2.
 \label{eq:generic-update-conditions}
\end{equation}
For every deterministic $\alpha\in(b,1]$,
\begin{equation}
 \frac1T\sum_{t=1}^T\E\norm{g_t}_{\psum}
 \le \frac{\num_T(K;\alpha)}{\alpha-b}.
 \label{eq:stationarity-main}
\end{equation}
For ordinary Smooth-NS DP-Muon take $K=1$ and
$b_{\rm mu}=\max_\ell\nu^2M_{2,\ell}/2$.
For strength-$\lambda$ DP-MuonBC take $K=K_\lambda$ and
\begin{equation}
 b_\lambda=\max_\ell
 \left\{|1-\lambda|\frac{\nu^2M_{2,\ell}}2
              +\lambda\frac{\nu^4M_{4,\ell}}4\right\}.
 \label{eq:multiblock-bias}
\end{equation}
\end{theorem}
The proof combines conditional alignment, martingale momentum tracking, and
the product-space descent inequality. The theorem allows coupled matrix
blocks and substantial clipping. At $\alpha=1$, substitute
\eqref{eq:gap-free-quality}; no lower bound on nonzero singular values is
needed. The condition $b<\alpha$ is a sufficient small-bias condition for
this particular certificate, not a necessary condition for convergence.
Appendix~\ref{app:ordinary} also gives an ordinary DP-Muon bound using its
\emph{noisy} momentum directly, without an output-bias denominator.

\begin{corollary}[Exact comparison of the specified upper bounds]
\label{cor:comparison}
Take common deterministic clipping and quality envelopes valid along both
trajectories, the same initialization and public hyperparameters, and
$\lambda=1$. In \eqref{eq:certificate-numerator} use these common envelopes.
Let $K=K_1$ and suppose $b_{\rm bc}<b_{\rm mu}<\alpha$, where
$b_{\rm bc}=\max_\ell\nu^4M_{4,\ell}/4$.
Then the DP-MuonBC bound in \eqref{eq:stationarity-main} is strictly smaller
than the corresponding bias-aware DP-Muon bound if and only if
\begin{equation}
 L\eta(K-1)\left[(\alpha+1)\frac{\beta}{1-\beta}+\frac{K+1}{2}\right]
 <\frac{b_{\rm mu}-b_{\rm bc}}{\alpha-b_{\rm mu}}\num_T(1;\alpha).
 \label{eq:dominance-criterion}
\end{equation}
\end{corollary}
This compares two explicitly derived \emph{upper bounds}. It is not an
ordering of realized losses, not a lower bound for DP-Muon, and not a
comparison against every available DP-Muon certificate.
A pair of big-$O$ upper bounds alone cannot establish a strictly faster
actual convergence rate. The criterion instead quantifies when the bias
improvement exceeds the correction's additional movement and smoothness cost.
It explains why a useful bias correction need not yield a large or uniform
empirical improvement.

\paragraph{The hybrid optimizer.}
When $u$ also changes, it affects the matrix gradients and the objective
change. Appendix~\ref{app:hybrid} proves a corresponding inequality with the
full update norm and an explicit adverse auxiliary-descent term. It also
bounds ordinary Adam's update magnitude under $\beta_1^2<\beta_2$.
These results do not make that adverse-descent term vanish or establish
stationarity of all Adam coordinates. Privacy covers the hybrid model;
the unconditional matrix-only theorem must not be described as a complete
hybrid convergence theorem.

\subsection{Fixed privacy versus fixed per-step noise}
For a chosen target budget, let
$\sigma_{\rm cert}(T,\pi,\eps,\delta)$ be a noise multiplier certified by
\eqref{eq:rdp-accountant}. Substituting
\begin{equation}
 \nu=\frac{C}{B}\sigma_{\rm cert}(T,\pi,\eps,\delta),\qquad
 v_C^2=\frac{(1-\pi)C^2}{B}
       +\frac{DC^2\sigma_{\rm cert}(T,\pi,\eps,\delta)^2}{B^2}
 \label{eq:privacy-utility}
\end{equation}
into the finite-horizon results gives the privacy--utility statement.
The same substitution must be made in the heat-bias, correction-moment,
and finite-step-error terms.
Increasing $T$ at fixed $\sigma$ is not a fixed-total-privacy experiment.
Appendix~\ref{app:fixed-privacy} gives a sufficient non-subsampled Gaussian
calibration and explains the growing-data or shrinking-sensitivity scaling
needed by that construction for $\nu_T\to0$.
For comparison, Appendix~\ref{app:ordinary} proves a fixed-noise DP-Muon
rate with explicitly increasing NS iteration count; it is labeled an
optimization asymptotic, not fixed-budget privacy convergence.

\section{Experiments}
\label{sec:experiments}

\paragraph{Setup.}
We evaluate differentially private fine-tuning of GPT-2 on the E2E
data-to-text benchmark \citep{novikova2017e2e}.
All optimizers use the same single-global per-example clipping mechanism:
the full trainable-parameter gradient is clipped once, followed by one
joint Gaussian release per update. We use Poisson sampling, 10 training
epochs, maximum sequence length 100, physical batch size 16, gradient
accumulation 64, and expected lot size 1,024, with
$\delta=8\times10^{-6}$.
The target privacy budgets
$\varepsilon\in\{1,2,4,8\}$ correspond to recorded values
$0.998599$, $1.995747$, $3.988708$, and $7.978774$.
Learning rates are selected using validation NLL, and the resulting
test NLL is reported.

For the Muon-based methods, hidden matrix parameters are updated by Muon,
while excluded vectors, embeddings, and the language-model head are
updated by auxiliary Adam.
DP-Muon uses the practical five-step Muon map without bias correction.
Throughout the experiments, \textbf{DP-MuonBC} denotes the selected
first-order Smooth-NS inverse-heat variant with correction strength
$\lambda=1.25$, one antithetic probe, $\kappa=2$, and smooth-normalization
parameter $\gamma_{\mathrm{ns}}=10^{-4}$.
The correction is applied only after the private gradient release and
therefore incurs no additional privacy cost.

\begin{table}[t]
\centering
\small
\begin{tabular}{lrrrr}
\toprule
Method
& $\varepsilon\approx1$
& $\varepsilon\approx2$
& $\varepsilon\approx4$
& $\varepsilon\approx8$\\
\midrule
DP-SGD$^\dagger$
    & 0.610520 & 0.635502 & 0.599148 & 0.572277\\
DP-Adam
    & 0.570687 & 0.548685 & 0.533640 & 0.519461\\
DP-AdamBC
    & 0.571525 & 0.547806 & 0.530832 & 0.515512\\
DP-Muon
    & 0.567287 & 0.543061 & 0.526304 & 0.513344\\
DP-MuonBC
    & \textbf{0.566709}
    & \textbf{0.542657}
    & \textbf{0.525857}
    & \textbf{0.512834}\\
\bottomrule
\end{tabular}
\caption{
E2E test NLL under single-global clipping; lower is better.
Hyperparameters are selected by validation NLL, and the table reports
the common seed-0 comparison across all five methods.
DP-MuonBC denotes first-order Smooth-NS correction with
$\lambda=1.25$.
$^\dagger$The best measured DP-SGD learning rate lies at the upper boundary
of its search range.
}
\label{tab:main-results}
\end{table}

\paragraph{Results.}
DP-Muon outperforms both DP-Adam and DP-AdamBC at every reported privacy
budget, showing that the advantage of the Muon update remains after using
the same global clipping mechanism for all optimizers.
DP-MuonBC achieves the lowest test NLL at all four privacy budgets.
Relative to DP-Muon, its NLL improvements are approximately
$5.8\times10^{-4}$, $4.0\times10^{-4}$,
$4.5\times10^{-4}$, and $5.1\times10^{-4}$
at $\varepsilon\approx1,2,4,8$, respectively.
Thus the empirical effect of bias correction is modest but consistently
positive on the primary likelihood metric in the common seed-0 comparison.

Additional generation metrics, learning-rate selections,
bias-correction ablations, and paired multi-seed comparisons are reported
in Appendix~\ref{app:experiments-details}.

\section{Conclusion}
DP-Muon uses a standard globally clipped private gradient followed by
matrix-valued post-processing. DP-MuonBC cancels the leading fresh-noise
heat-smoothing term at the correctly calibrated map input. The accompanying
variance and update-moment analysis identifies when the derived stationarity
certificate improves.

\newpage

\bibliographystyle{plainnat}
\bibliography{references}

@inproceedings{abadi2016deep,
  title={Deep learning with differential privacy},
  author={Abadi, Mart{\'\i}n and Chu, Andy and Goodfellow, Ian and McMahan, H Brendan and Mironov, Ilya and Talwar, Kunal and Zhang, Li},
  booktitle={Proceedings of the 2016 ACM SIGSAC Conference on Computer and Communications Security},
  pages={308--318},
  year={2016}
}

@article{dwork2014algorithmic,
  title={The algorithmic foundations of differential privacy},
  author={Dwork, Cynthia and Roth, Aaron},
  journal={Foundations and Trends in Theoretical Computer Science},
  volume={9},
  number={3--4},
  pages={211--407},
  year={2014},
  publisher={Now Publishers, Inc.}
}

@misc{jordan2024muon,
  title={Muon: An optimizer for hidden layers in neural networks},
  author={Jordan, Keller and Jin, Yuchen and Boza, Vlado and You, Jiacheng and Cesista, Franz and Newhouse, Laker and Bernstein, Jeremy},
  year={2024},
  url={https://github.com/KellerJordan/Muon}
}

@inproceedings{kim2026convergence,
  author    = {Gyu Yeol Kim and {Min-hwan} Oh},
  title     = {{Convergence of Muon with Newton--Schulz}},
  booktitle = {International Conference on Learning Representations},
  year      = {2026}
}

@inproceedings{kingma2015adam,
author       = {Diederik P. Kingma and
                  Jimmy Ba},
  title        = {Adam: {A} Method for Stochastic Optimization},
  booktitle    = {3rd International Conference on Learning Representations, {ICLR} 
                  },
  year         = {2015},
    url       = {https://arxiv.org/abs/1412.6980}
}

@inproceedings{li2021large,
  title={Large language models can be strong differentially private learners},
  author={Li, Xuechen and Tram{\`e}r, Florian and Liang, Percy and Hashimoto, Tatsunori},
  booktitle={International Conference on Learning Representations},
  year={2022}
}

@inproceedings{mironov2017renyi,
  title={R{\'e}nyi differential privacy},
  author={Mironov, Ilya},
  booktitle={2017 IEEE 30th Computer Security Foundations Symposium (CSF)},
  pages={263--275},
  year={2017},
  organization={IEEE}
}

@article{mironov2019sampled,
author       = {Ilya Mironov and
                  Kunal Talwar and
                  Li Zhang},
  title        = {R{\'{e}}nyi Differential Privacy of the Sampled Gaussian Mechanism},
  journal      = {CoRR},
  volume       = {abs/1908.10530},
  year         = {2019},
}

@inproceedings{novikova2017e2e,
  title={{The E2E dataset: New challenges for end-to-end generation}},
  author={Novikova, Jekaterina and Du{\v{s}}ek, Ond{\v{r}}ej and Rieser, Verena},
  booktitle={Proceedings of the 18th annual SIGdial meeting on discourse and dialogue},
  pages={201--206},
  year={2017}
}

@article{pichapati2019adaclip,
  author       = {Venkatadheeraj Pichapati and
                  Ananda Theertha Suresh and
                  Felix X. Yu and
                  Sashank J. Reddi and
                  Sanjiv Kumar},
  title        = {{AdaCliP: Adaptive Clipping for Private {SGD}}},
  journal={arXiv preprint arXiv:1908.07643},
  year={2019}
}

@article{shen2025,
  author  = {Wei Shen and Ruichuan Huang and Minhui Huang and Cong Shen and Jiawei Zhang},
  title   = {{On the Convergence Analysis of Muon}},
  journal = {arXiv preprint arXiv:2505.23737},
  year    = {2025},
}

@inproceedings{DPAdamBC2024,
  author       = {Qiaoyue Tang and
                  Frederick Shpilevskiy and
                  Mathias L{\'{e}}cuyer},
  title        = {{DP-AdamBC: Your DP-Adam Is Actually {DP-SGD} (Unless You Apply Bias
                  Correction)}},
  booktitle    = {Thirty-Eighth {AAAI} Conference on Artificial Intelligence, {AAAI}
                  2024},
  pages        = {15276--15283},
  publisher    = {{AAAI} Press},
  year         = {2024},
}

@inproceedings{wang2019subsampled,
  title={Subsampled R{\'e}nyi differential privacy and analytical moments accountant},
  author={Wang, Yu-Xiang and Balle, Borja and Kasiviswanathan, Shiva Prasad},
  booktitle={The 22nd International Conference on Artificial Intelligence and Statistics},
  pages={1226--1235},
  year={2019},
  organization={PMLR}
}

@inproceedings{yu2021large,
  title={Large scale private learning via low-rank reparametrization},
  author={Yu, Da and Zhang, Huishuai and Chen, Wei and Yin, Jian and Liu, Tie-Yan},
  booktitle={International Conference on Machine Learning},
  pages={12208--12218},
  year={2021},
  organization={PMLR}
}

\appendix

\section{Privacy of the global release}
\label{app:privacy}
\begin{proof}[Proof of Theorem~\ref{thm:dp-muon-privacy}]
Condition on previous \emph{private releases} and data-independent optimizer
randomness. The current parameters are fixed functions of these quantities.
This conditioning does not reveal previous Gaussian privacy-noise draws or
unnoised gradients. For each record, \eqref{eq:global-clip} gives
$\norm{H_{t,i}}_2\le C$ simultaneously over all released coordinates.
The base query $B^{-1}\sum_i H_{t,i}$ therefore has add/remove sensitivity
at most $C/B$, with $B$ kept fixed across adjacent datasets.
The noise covariance in \eqref{eq:global-release} is
$(\sigma C/B)^2I$, so independent Bernoulli inclusion followed by this
release is the Poisson-sampled Gaussian mechanism with multiplier $\sigma$.

By definition of the valid sampled-Gaussian RDP bound, its conditional
order-$\alpha$ privacy cost is at most
$\rho_{\rm SGM}(\alpha;\pi,\sigma)$ for every possible preceding private
transcript. Adaptive composition sums these bounds over the $T$ releases.
Conversion from $(\alpha,\rho)$-RDP to
$(\rho+\log(1/\delta)/(\alpha-1),\delta)$-DP gives
\eqref{eq:rdp-accountant}; see \citet{mironov2017renyi,mironov2019sampled}.
An infimum over a finite accountant grid is also valid. No claim of exactness
or numerical tightness of this conversion is needed.

Momentum, the matrix maps, Gaussian \emph{probes}, auxiliary Adam state,
public learning-rate rescaling, and parameter updates use only already
privatized coordinates and independent randomness. They are randomized
post-processing. In particular, the probes are not additional training-data
queries. The same conditional sensitivity and RDP argument holds at each
new adaptive state, even when changing the optimizer changes all later
queried gradients. This proves the certificate for the final full model and
any intermediate quantities that are themselves post-processing of these
releases.
\end{proof}

\paragraph{What is not covered by post-processing.}
The unnoised $\bar G_t$ or $A_t$, per-example gradients or norms, sampled
record identities, and Gaussian \emph{privacy-noise} draws are not made
public by this theorem. For example, releasing $Z_t$ together with
$\widetilde G_t$ would expose $\bar G_t$. Optimizer probes are distinct
independent random variables and can be public. If validation data are
sensitive, their queries require privacy accounting as well.
The mathematical mechanism assumes exact sensitivity control, Gaussian
noise, and the stated sampler; it does not certify a particular
floating-point noise generator or a sampler implementation from a table
of aggregate results.

For replace-one adjacency the base sensitivity is at most $2C/B$.
An accountant for the corresponding sampling and adjacency model must
then be used; a numerical certificate for a different adjacency model is
not automatically interchangeable. Deterministic recordwise preprocessing
can be incorporated into the per-record query, but data-dependent changes
to normalization, sampling rates, record grouping, or the definition of a
record must also match the privacy model.

\paragraph{Hyperparameter search.}
The theorem applies to a configured run or to a sequence of adaptively
chosen releases whose complete privacy cost is accounted for. Repeated
candidate training runs on the same private data are repeated accesses,
even when final selection evaluates public validation examples.
A per-run $\eps$ does not automatically certify a pipeline that searches
and releases information about many runs. The experimental tables report
the configured-run budget, not a privacy guarantee for the full research
search history.

\section{Smooth Newton--Schulz maps and gap-free quality control}
\label{app:ns}
\subsection{Canonical partial polar factor and regularity}
For a thin SVD $A=U\operatorname{diag}(s_i)V^\top$, with nonnegative
singular values, define the \emph{canonical partial polar factor}
\begin{equation}
 \Polar(A)=U\operatorname{diag}(\mathbf 1_{\{s_i>0\}})V^\top.
 \label{eq:partial-polar}
\end{equation}
It is zero on the null singular subspaces, is invariant to positive
rescaling, and satisfies
$\norm{\Polar(A)}_\op\le1$ and
$\inner{A}{\Polar(A)}_F=\norm{A}_*$.
This convention matters: an arbitrary orthogonal completion on a nullspace
need not be approximated by the NS map.

\begin{lemma}[Regularity, boundedness and derivative scaling]
\label{lem:regularity}
For fixed $m,n,\kappa,q$ and $\gamma>0$, the map
$Q_\gamma=Q_{\gamma,\kappa,q}$ is $C^\infty$, has bounded derivatives
of every fixed order, and satisfies $\sup_A\norm{Q_\gamma(A)}_\op\le1$.
For every integer $s\ge1$ there is a finite $C_s$, independent of $\gamma$,
such that
\begin{equation}
 \sup_X\norm{D^s Q_\gamma(X)}_{F\leftarrow F^s}\le C_s\gamma^{-s}.
 \label{eq:derivative-scaling}
\end{equation}
Consequently, $H_2\le C_2\gamma^{-2}$ and
$M_{2j}\le d^j C_{2j}\gamma^{-2j}$.
These constants may depend on the shape and on the complete finite-step map;
there is no assertion that they remain bounded as $q$ or dimension grows.
\end{lemma}
\begin{proof}
The normalization $N_\gamma(X)=\mathcal T(X)/
\sqrt{\gamma^2+\norm{X}_F^2}$ is smooth, with Frobenius norm less than one.
For $0<s\le1$, the binomial series gives
\[
 \frac1s=\sum_{j=0}^\infty c_j(1-s^2)^j,
 \qquad c_j\ge0.
\]
Thus $0\le s p_\kappa(s^2)\le1$, also at $s=0$ by continuity.
Each NS step applies this scalar response to the singular values, so all
iterates have operator norm at most one and Frobenius norm at most
$\sqrt{\min(m,n)}$. This proves the output bound.

The map $N_1$ has uniformly bounded derivatives of each fixed order:
an order-$s$ derivative is a finite sum of rational terms with denominator
$(1+\norm{X}_F^2)^{1/2+k}$ and numerator polynomial of degree at most
$1+2k-s$; terms with a negative degree are zero. Each term is bounded
near the origin and at infinity. Composing $N_1$ with a fixed polynomial map a finite
number of times retains bounded derivatives, since $N_1$ ranges in a
bounded ball and all polynomial derivatives are bounded on its closure.
Also $Q_\gamma(X)=Q_1(X/\gamma)$ exactly. Differentiating this identity
proves \eqref{eq:derivative-scaling}. The Laplacian power $\Delta^j$
is a sum of $d^j$ derivatives of total order $2j$; the operator norm is
bounded by the Frobenius norm. The claimed Laplacian envelope follows.
\end{proof}

\subsection{Scalar residual contraction}
Let
\begin{equation}
 P_\kappa(z)=\sum_{j=0}^{\kappa}c_jz^j,
 \qquad \phi_\kappa(x)=xP_\kappa(1-x^2),\qquad 0\le x\le1.
 \label{eq:scalar-ns}
\end{equation}
This $P_\kappa$ is a polynomial in the residual; the main-text
$p_\kappa(z)$ equals $P_\kappa(1-z)$.

\begin{lemma}[A scalar identity with nonnegative coefficients]
\label{lem:scalar-residual}
For $z\in[0,1]$,
\begin{equation}
 1-(1-z)P_\kappa(z)^2
 =(2\kappa+1)c_\kappa\sum_{j=0}^{\kappa}
             \frac{c_j}{\kappa+j+1}z^{\kappa+j+1}
 \le z^{\kappa+1}.
 \label{eq:scalar-residual}
\end{equation}
The coefficients on the middle expression are nonnegative and sum to one.
Moreover,
$\phi_\kappa'(x)=(2\kappa+1)c_\kappa(1-x^2)^\kappa\ge0$.
\end{lemma}
\begin{proof}
The coefficients satisfy $(j+1)c_{j+1}=(j+\tfrac12)c_j$. Therefore
\begin{equation}
 (1-z)P_\kappa'(z)
 =\frac12P_\kappa(z)-\frac{2\kappa+1}{2}c_\kappa z^\kappa.
 \label{eq:coefficient-recurrence}
\end{equation}
The derivative of $R_\kappa(z)=1-(1-z)P_\kappa(z)^2$ is then
$(2\kappa+1)c_\kappa z^\kappa P_\kappa(z)$.
Integrating from zero, where $R_\kappa(0)=0$, proves the identity.
At $z=1$ its value is one. All coefficients are nonnegative, and every
power is at least $\kappa+1$, proving the inequality on $[0,1]$.
Differentiating $xP_\kappa(1-x^2)$ and applying
\eqref{eq:coefficient-recurrence} proves the last assertion.
\end{proof}

\begin{lemma}[Pointwise polar error and weighted alignment defect]
\label{lem:gap-free-defect}
Let $r=\min(m,n)$, $c(A)=\sqrt{\gamma^2+\norm{A}_F^2}$,
and $p_q=(\kappa+1)^q$. For every singular value $s_i\ge0$ of $A$, set
$x_{i,0}=s_i/c(A)$ and $x_{i,q}=\phi_\kappa^{\circ q}(x_{i,0})$.
Then
\begin{equation}
 0\le1-x_{i,q}^2\le(1-x_{i,0}^2)^{p_q}.
 \label{eq:pointwise-residual}
\end{equation}
For $A\ne0$, define
$\delta_0(A)=\max_{s_i>0}(1-s_i^2/c(A)^2)$.
With the zero convention at $A=0$,
\begin{equation}
 \norm{Q(A)-\Polar(A)}_\op
 \le 1-\sqrt{1-\delta_0(A)^{p_q}}.
 \label{eq:pointwise-polar}
\end{equation}
In addition, for every $A$, including zero and rank-deficient matrices,
\begin{align}
 e_q(A)&:=\norm{A}_*-\inner{A}{Q(A)}_F
       =\sum_{i=1}^r s_i(1-x_{i,q})\ge0, \label{eq:exact-defect}\\
 e_q(A)&\le\frac{r\,c(A)}{\sqrt{2e p_q}}.
 \label{eq:uniform-defect}
\end{align}
Zero singular values contribute zero to the sum.
\end{lemma}
\begin{proof}
The singular-value recursion and Lemma~\ref{lem:scalar-residual} imply
\eqref{eq:pointwise-residual} inductively. All scalar responses lie in
$[0,1]$, so \eqref{eq:pointwise-polar} follows on the nonzero singular
subspaces. Both maps vanish on the null singular subspaces by
\eqref{eq:partial-polar}.
For each normalized singular value $x=s_i/c(A)$,
\[
 1-x_{i,q}\le1-x_{i,q}^2
 \le(1-x^2)^{p_q}\le e^{-p_q x^2}.
\]
The maximum of $x e^{-p_q x^2}$ over $x\ge0$ is
$1/\sqrt{2e p_q}$, attained at $x=1/\sqrt{2p_q}$.
Multiplying by $c(A)$ and summing at most $r$ terms yields
\eqref{eq:uniform-defect}. No minimum nonzero singular value enters this
last bound.
\end{proof}

\begin{proof}[Proof of the deterministic expected bound \eqref{eq:gap-free-quality}]
Use a concatenated Frobenius norm over the matrix coordinates.
Because $\norm{H_{t,i}^W}_F\le C$, conditional on the preceding state,
\[
 \E[\norm{\bar G_t^W}_F\mid\F_{t-1}]
 \le\frac1B\sum_{i=1}^N\pi\norm{H_{t,i}^W}_F\le C.
\]
This is an \emph{expectation} bound; for a Poisson lot, the realized
$\norm{\bar G_t^W}_F$ need not be at most $C$.
Also $\E\norm{Z_t^W}_F\le\nu\sqrt D$.
The normalized pre-noise momentum has the convex-weight expansion
\begin{equation}
 \bar A_t=\sum_{j<t}w_{t,j}\widetilde G_j^W+w_{t,t}\bar G_t^W,
 \label{eq:pre-expansion}
\end{equation}
so $\E\norm{\bar A_t}_F\le C+\nu\sqrt D$.
At $\alpha=1$, the positive parts in \eqref{eq:quality-defect} equal the
nonnegative $e_q(A_t^{(\ell)})$. Lemma~\ref{lem:gap-free-defect} gives
\begin{align*}
 \frac1{s_t}\sum_\ell e_q(A_t^{(\ell)})
 &\le\frac1{\sqrt{2e p_q}}
       \sum_\ell r_\ell\left(\frac{\gamma_\ell}{s_t}
                              +\norm{\bar A_t^{(\ell)}}_F\right)\\
 &\le\frac{R}{\sqrt{2e p_q}}
                 (\gamma_{\max}+\norm{\bar A_t}_F),
\end{align*}
using $s_t\ge1$. Taking expectations and averaging proves
\eqref{eq:gap-free-quality}. The same calculation applies to $M_t/s_t$
when all terms in its convex expansion include their Gaussian noise.
\end{proof}

\paragraph{Why a measured maximum is not an unconditional constant.}
For a realized finite trajectory, \eqref{eq:pointwise-polar} produces a
finite path-dependent factor. Its realized maximum may vary arbitrarily
across trajectories. Positivity of all observed singular values does not
make this factor deterministic, uniformly integrable, or bounded uniformly
in the horizon. A theorem using a deterministic almost-sure relative-quality
envelope may state that envelope as an assumption. Our default result
instead keeps the expected weighted defect and uses
\eqref{eq:gap-free-quality}. Conditioning an entire martingale analysis on
an observed ``good-quality'' event would itself require justification and
is not used here.

\section{Adaptive heat identities and their limitations}
\label{app:heat}
\subsection{Conditional Gaussian representation and heat remainder}
\begin{lemma}[The current noise, not a frozen path]
\label{lem:fresh-gaussian}
Under the actual momentum recursion, $A_t$ in
\eqref{eq:pre-noise-state} is $\G_t$-measurable and
$M_t\mid\G_t$ has distribution $A_t+\nu E_t$, with $E_t$ standard
Gaussian. This statement permits arbitrary adaptation of earlier matrix
and auxiliary parameters.
\end{lemma}
\begin{proof}
The recursion is $M_t=\beta M_{t-1}+\bar G_t+Z_t$.
The first two summands are $\G_t$-measurable, whereas $Z_t$ is independent
of $\G_t$ with covariance $\nu^2I$. Nothing in this calculation treats
past clipped gradients as independent of past privacy noise.
\end{proof}

For a bounded matrix-valued map, Gaussian averaging is a contraction in the
uniform operator norm:
\begin{equation}
 \sup_X\norm{\mathsf S_sF(X)}_\op
 \le\sup_X\norm{F(X)}_\op,\qquad
 \mathsf S_s=\Heat_{\sqrt{2s}}.
 \label{eq:semigroup-contraction}
\end{equation}
For the maps of Lemma~\ref{lem:regularity}, bounded derivatives justify
commuting the Laplacian and heat averaging and integrating the heat
equation $\partial_s\mathsf S_sQ=\mathsf S_s\Delta Q$.
Repeated integration gives \eqref{eq:heat-remainder}. In particular,
\begin{equation}
 \sup_X\norm{(\Heat_\nu-I)^jQ(X)}_\op
       \le(\nu^2/2)^jM_{2j}.
 \label{eq:iterated-heat-bound}
\end{equation}
All uses here are for $C^\infty$ maps with the necessary bounded derivatives;
boundedness of an isolated distributional Laplacian is not being asserted
to suffice for the differentiations.

\subsection{First-order and arbitrary-order correction}
\begin{proof}[Proof of Theorem~\ref{thm:adaptive-bias}]
The antithetic estimator has conditional mean
$\E[\Hhat_{\nu,J}Q(M_t)\mid\G_t,M_t]=\Heat_\nu Q(M_t)$.
By Lemma~\ref{lem:fresh-gaussian} and independent Gaussian addition,
\[
 \E[\Hhat_{\nu,J}Q(M_t)\mid\G_t]
 =\Heat_\nu^2 Q(A_t)=\Heat_{\sqrt2\nu}Q(A_t).
\]
Consequently
\[
 \E[O_t^{\lambda,J}\mid\G_t]-Q(A_t)
 =\big[(1+\lambda)\Heat_\nu-\lambda\Heat_\nu^2-I\big]Q(A_t).
\]
Substitute $\Heat_\nu=I+D_\nu$ and expand to obtain
\eqref{eq:strength-identity}. The triangle inequality and
\eqref{eq:iterated-heat-bound} for $j=1,2$ prove
\eqref{eq:strength-bias}. Probe count does not enter the conditional mean.
\end{proof}

\begin{theorem}[Arbitrary-order identity and deterministic magnitude]
\label{thm:order-appendix}
The direction \eqref{eq:hierarchy} has the conditional mean and bias bound
\eqref{eq:hierarchy-bias}, and every realization satisfies
\begin{equation}
 \norm{O_t^{[K]}}_\op\le 2^{K+1}-1.
 \label{eq:order-magnitude}
\end{equation}
\end{theorem}
\begin{proof}
For each heat level $j$, independent Gaussian addition gives
\[
 \E[\Hhat_{\sqrt j\nu,J_j}Q(M_t)\mid\G_t]
 =\Heat_\nu^{j+1}Q(A_t).
\]
The finite polynomial identity
\[
 \sum_{j=0}^K(-1)^j\binom{K+1}{j+1}z^{j+1}
      =1-(1-z)^{K+1}
\]
proves the asserted mean after substituting $z=\Heat_\nu$.
Equation~\eqref{eq:iterated-heat-bound} bounds its bias.
Each heat estimate is an average of matrices of operator norm at most one,
and
$\sum_{j=0}^K|c_{K,j}|=\sum_{j=1}^{K+1}\binom{K+1}{j}=2^{K+1}-1$.
The triangle inequality proves \eqref{eq:order-magnitude}.
Correlation between different heat-level estimators does not change these
mean and deterministic-bound arguments.
\end{proof}
The $K=0,1,2$ conditional bias bounds are $aM_2,a^2M_4,a^3M_6$,
respectively. These are bounds relative to the same adaptive clean map,
not statements of equal stochastic variance or equal training trajectories.
The larger coefficient norm is a possible cost of higher order, not a
proof that its actual variance must always increase.

\subsection{Wrong probe scales and normalized momentum}
\begin{proof}[Proof of Proposition~\ref{prop:mismatch}]
As above, the probe mean is $\Heat_\rho Q(M_t)$; averaging over the
current privacy noise yields \eqref{eq:mismatch-exact}.
Write $\mathsf S_a=\Heat_\nu$ and $\mathsf S_c=\Heat_\rho$.
The bias can be written as
\[
 (\mathsf S_a-I)Q-\lambda\mathsf S_a(\mathsf S_c-I)Q.
\]
The heat equation gives the Taylor remainder bounds
\begin{align*}
 \norm{(\mathsf S_a-I)Q-a\Delta Q}_{\infty,\op}
 &\le\tfrac12a^2M_4,\\
 \norm{(\mathsf S_c-I)Q-c\Delta Q}_{\infty,\op}
 &\le\tfrac12c^2M_4,\\
 \norm{(\mathsf S_a-I)\Delta Q}_{\infty,\op}
 &\le aM_4.
\end{align*}
Here $\norm{F}_{\infty,\op}=\sup_X\norm{F(X)}_\op$.
Decompose
\[
 \mathsf S_a(\mathsf S_c-I)Q
 =c\Delta Q+c(\mathsf S_a-I)\Delta Q
       +\mathsf S_a[(\mathsf S_c-I)Q-c\Delta Q].
\]
Using contraction and collecting the linear term proves
\eqref{eq:mismatch-remainder}. At the matched scale a sharper cancellation
is available from Theorem~\ref{thm:adaptive-bias}; the generic remainder
bound is not claimed to be sharp there.
\end{proof}

Suppose an implementation applies a fixed smooth map to
$Y_t=M_t/s_t$. Conditionally on $\G_t$,
\begin{equation}
 Y_t=A_t/s_t+(\nu/s_t)E_t.
 \label{eq:normalized-fresh}
\end{equation}
For constant $\nu$ and $\beta$,
\begin{equation}
 V_t=\nu^2\sum_{j=0}^{t-1}\beta^{2j},\qquad
 \frac{(\sqrt{V_t}/s_t)^2}{(\nu/s_t)^2}
 =\sum_{j=0}^{t-1}\beta^{2j}.
 \label{eq:scale-ratio}
\end{equation}
The ratio tends to $1/(1-\beta^2)$, rather than one, when $\beta>0$.
The marginal accumulated Gaussian-noise variance is therefore not the
conditional variance in \eqref{eq:normalized-fresh}.
At $\lambda=1$, using it as a probe variance generally leaves a second-order
term. More generally, the leading coefficient in the mismatch theorem
vanishes when $\lambda\rho^2=\nu_{\rm input}^2$; neither a tuned strength
nor an accumulated scale should be described as matched without checking
that equality and the rest of the assumptions.

A deterministic positive scaling can alternatively be absorbed into a
time-dependent map $Q_t(X)=Q(X/s_t)$. The heat argument then uses the raw
fresh scale $\nu$ and the derivatives of $Q_t$, equivalently the normalized
scale $\nu/s_t$ and derivatives of $Q$. A normalization depending on the
current noisy input must remain part of $Q$; it cannot be conditioned away
as though it were independent of the current noise.

\subsection{Bias reduction alone does not order optimization}
\label{app:no-dominance}
\begin{proposition}[A limitation of a mean-bias argument]
\label{prop:no-universal-gain}
For fixed $\lambda>0$ there is no universal one-step gradient-alignment
improvement that follows from the heat identity alone for all smooth
bounded maps, states, and gradients.
\end{proposition}
\begin{proof}
Take a scalar map $Q(x)=\cos x$, state $A=0$, gradient $g=-1$, and
$a=\nu^2/2>0$. Then
$\Heat_\nu Q(0)=e^{-a}$ and $\Heat_\nu^2Q(0)=e^{-2a}$.
The corrected mean minus the uncorrected mean is
$\lambda(e^{-a}-e^{-2a})>0$. Multiplying by $g=-1$ makes its gradient
alignment strictly worse. This is a counterexample to a claim based only
on the heat identity, not a counterexample to a theorem that additionally
assumes the Smooth-NS geometry and the quantitative descent conditions.
\end{proof}

\section{Finite-probe and update-moment proofs}
\label{app:moments}
\begin{lemma}[Antithetic second difference]
\label{lem:second-difference}
If $H_2$ is the Hessian envelope in \eqref{eq:hessian-envelope}, then
\begin{equation}
 \norm{Q(X)-\tfrac12[Q(X+\nu U)+Q(X-\nu U)]}_F
       \le\frac{H_2\nu^2}{2}\norm{U}_F^2.
 \label{eq:second-difference}
\end{equation}
\end{lemma}
\begin{proof}
Taylor's formula with integral remainder states
\[
 Q(X\pm h)=Q(X)\pm DQ(X)[h]
 +\int_0^1(1-s)D^2Q(X\pm sh)[h,h]ds.
\]
Add the two expressions, divide by two, and use
$\int_0^1(1-s)ds=1/2$, with $h=\nu U$.
The two first derivatives cancel; the total bound is $H_2\norm h_F^2/2$.
\end{proof}

\begin{proof}[Proof of Proposition~\ref{prop:moments}]
For a fixed matrix center, let
$C_{\nu,J}(X)=Q(X)-\Hhat_{\nu,J}Q(X)$ and
$S_J=J^{-1}\sum_{j=1}^J\norm{U_j}_F^2$.
Lemma~\ref{lem:second-difference} implies
\[
 \norm{C_{\nu,J}(X)}_F\le\frac{H_2\nu^2}{2}S_J.
\]
The independent standard Gaussian entries give
$\E S_J=d$ and $\E S_J^2=d^2+2d/J$.
Consequently
\begin{equation}
 \E\norm{C_{\nu,J}(X)}_F^2
 \le\left(\frac{H_2\nu^2}{2}\right)^2(d^2+2d/J).
 \label{eq:increment-moment}
\end{equation}
These bounds are uniform in $X$, and hence hold conditionally at a noisy
center independent of the current probes.

Let $C_{t,\ell}$ denote the blockwise correction increment and set
$S_t^C=(\sum_\ell\norm{C_{t,\ell}}_F^2)^{1/2}$.
Then
\[
 \norm{O_t}_{\pmax}\le1+\lambda S_t^C,\qquad
 \E(S_t^C)^2\le\Xi_{\rm all}^2.
\]
Cauchy--Schwarz gives the first-moment bound
$1+\lambda\Xi_{\rm all}$, and expansion of the square gives the
second-moment bound $(1+\lambda\Xi_{\rm all})^2$.
Separately, each realization has norm at most
$(1+\lambda)+\lambda=1+2\lambda$, from the unit norm bound on the
base map and on its heat estimates. Taking the smaller of the two bounds
proves \eqref{eq:update-moments}. Independence between distinct blocks
is not needed for this argument.

For the Monte Carlo variance, let
$Y_j=[Q(X+\nu U_j)+Q(X-\nu U_j)]/2$.
They are independent, identically distributed conditional on $X$, with
mean $\Heat_\nu Q(X)$ and $\norm{Y_j}_F^2\le r$.
The variance of their average is at most $r/J$.
Alternatively, centering a random variable at its expectation minimizes
mean squared distance, so
\[
 \E\norm{Y_j-\E Y_j}_F^2
 \le\E\norm{Y_j-Q(X)}_F^2
 \le\frac{H_2^2\nu^4}{4}\E\norm{U_j}_F^4
 =\frac{H_2^2\nu^4}{4}d(d+2).
\]
Divide by $J$ and combine the two bounds to obtain
\eqref{eq:probe-variance}.
\end{proof}

The $J^{-1}$ variance term is different from the nonzero correction itself:
\[
 \E[C_{\nu,J}(X)\mid X]=Q(X)-\Heat_\nu Q(X).
\]
Even infinitely many probes do not remove that correction mean or the
randomness of $M_t$. There is no deterministic near-unit norm guarantee
in Proposition~\ref{prop:moments}; its deterministic fallback is
$1+2\lambda$. For $\lambda=1$, the fixed-shape, fixed-map small-noise
statement is $K_1-1=O(\nu^2)$, with all shape and derivative dependence
exposed in \eqref{eq:moment-constant}.

\section{Poisson oracle, momentum tracking, and stationarity}
\label{app:tracking}
\subsection{Norm duality and descent}
By blockwise nuclear--operator duality,
\begin{equation}
 |\inner{G}{H}|\le\norm{G}_{\psum}\norm{H}_{\pmax},\qquad
 \norm{G}_{\psum}\le\sqrt R\left(\sum_\ell\norm{G^{(\ell)}}_F^2\right)^{1/2}.
 \label{eq:product-duality}
\end{equation}
Choosing each $H^{(\ell)}$ as the corresponding canonical partial polar
factor attains the first dual supremum. The second inequality follows by
$\norm{G^{(\ell)}}_*\le\sqrt{r_\ell}\norm{G^{(\ell)}}_F$ and
Cauchy--Schwarz over $\ell$.
Assumption~\ref{ass:smoothness} implies, by integration along a line segment,
\begin{equation}
 f(W+H)\le f(W)+\inner{\nabla f(W)}{H}
                     +\frac L2\norm{H}_{\pmax}^2.
 \label{eq:product-descent}
\end{equation}
Indeed, the remainder integral is bounded by
$\int_0^1Ls\norm{H}_{\pmax}^2ds$.
The one-sided inequality alone is not asserted to be equivalent to the
full gradient Lipschitz assumption.

\subsection{Clipping and the exact Poisson variance}
Let $I_{t,i}=\mathbf1_{\{i\in\mathcal L_t\}}$ and define
\begin{equation}
 h_t=\frac1N\sum_i H_{t,i}^W,\qquad b_t=h_t-g_t,\qquad
 \xi_t=\frac1B\sum_i(I_{t,i}-\pi)H_{t,i}^W+Z_t^W.
 \label{eq:oracle-decomposition}
\end{equation}
All per-example gradients and clipping multipliers in this expression
are fixed after conditioning on $\F_{t-1}$.
The superscript $W$ selects the matrix coordinates of the global query.
Then $\widetilde G_t^W=g_t+b_t+\xi_t$ and
\begin{equation}
 \E[\xi_t\mid\F_{t-1}]=0,
 \quad \E\norm{b_t}_{\psum}\le\bias_T(C),
 \quad \E[\norm{\xi_t}_F^2\mid\F_{t-1}]\le v_C^2.
 \label{eq:oracle-bounds}
\end{equation}
For the bias, \eqref{eq:product-duality} gives
\[
 \norm{b_t}_{\psum}
 \le\frac1N\sum_i(1-a_{t,i})\sum_\ell\sqrt{r_\ell}
                                      \norm{G_{t,i}^{(\ell)}}_F
 \le\frac{\sqrt R}{N}\sum_i(1-a_{t,i})\norm{G_{t,i}}_2.
\]
For $x\ge0$, the stabilizer obeys
\[
 x\left(1-\min\{1,C/(x+\tau_c)\}\right)
       \le(x-C)_++\tau_c,
\]
with the specified zero convention. This proves the tail bound stated after
\eqref{eq:clip-residual}, and shows why a global norm rather than an
individual matrix norm controls the residual.

Conditional independence of the Bernoulli indicators gives the exact
sampling variance
\begin{equation}
 \E\left[\norm{\frac1B\sum_i(I_{t,i}-\pi)H_{t,i}^W}_F^2
                       \middle|\F_{t-1}\right]
 =\frac{\pi(1-\pi)}{B^2}\sum_i\norm{H_{t,i}^W}_F^2
 \le\frac{(1-\pi)C^2}{B}.
 \label{eq:poisson-variance}
\end{equation}
The current Gaussian noise is independent and adds $D\nu^2$ to the
squared-Frobenius second moment. This proves \eqref{eq:oracle-bounds}.
In particular, this proof does not replace the Poisson lot by an i.i.d.
fixed-size minibatch or use its random size in the denominator.

\subsection{Deterministic weights and adaptive martingale tracking}
\begin{lemma}[Pre-noise momentum tracking]
\label{lem:tracking}
Under Assumption~\ref{ass:smoothness}, the above oracle, and
$W_t=W_{t-1}-\eta O_t$ with
$\sup_t\E\norm{O_t}_{\pmax}\le K$, inequality
\eqref{eq:tracking-main} holds. Universal constants $c_1,c_2$ satisfy
\begin{equation}
 \Gamma_T(\beta)\le\frac{c_1}{T\sqrt{1-\beta}}+c_2\sqrt{1-\beta}.
 \label{eq:gamma-bound}
\end{equation}
\end{lemma}
\begin{proof}
Fix $t$ and expand \eqref{eq:pre-expansion}. For $j<t$ let
$\xi_{t,j}=\xi_j$, and for $j=t$ let
$\xi_{t,t}=B^{-1}\sum_i(I_{t,i}-\pi)H_{t,i}^W$, omitting the fresh noise.
Then
\begin{equation}
 \bar A_t-g_t
 =\sum_{j=1}^t w_{t,j}b_j
  +\sum_{j=1}^t w_{t,j}\xi_{t,j}
  +\sum_{j=1}^t w_{t,j}(g_j-g_t).
 \label{eq:tracking-decomposition}
\end{equation}
The expected dual norm of the first sum is at most $\bias_T(C)$, since
its weights are nonnegative and sum to one.
For $i<j$, the earlier fluctuation is $\F_{j-1}$-measurable and
$\E[\xi_{t,j}\mid\F_{j-1}]=0$. Hence
\[
 \E\inner{\xi_{t,i}}{\xi_{t,j}}_F=0.
\]
No independence of adaptive gradients across time is used. It follows that
\[
 \E\norm{\sum_{j=1}^t w_{t,j}\xi_{t,j}}_F^2
 \le v_C^2\sum_{j=1}^t w_{t,j}^2.
\]
Jensen's inequality and \eqref{eq:product-duality} bound the expected dual
norm of this sum by $\sqrt R\,v_C\sqrt{\sum_jw_{t,j}^2}$.
Finally,
\[
 \E\norm{g_j-g_t}_{\psum}
 \le L\E\norm{W_{j-1}-W_{t-1}}_{\pmax}
 \le L\eta K(t-j).
\]
For $\beta\in[0,1)$ the weighted age is
\begin{equation}
 \sum_{j=1}^t w_{t,j}(t-j)
 =\frac{\beta}{1-\beta}-\frac{t\beta^t}{1-\beta^t}
 \le\frac{\beta}{1-\beta}.
 \label{eq:weighted-age}
\end{equation}
The formula uses the natural zero convention when $\beta=0$.
Combining the three terms and averaging proves \eqref{eq:tracking-main}.

For the weight bound, summing two geometric series gives
\begin{equation}
 \sum_{j=1}^t w_{t,j}^2
 =\frac{1-\beta}{1+\beta}\frac{1+\beta^t}{1-\beta^t}.
 \label{eq:weight-squares}
\end{equation}
Let $d_\beta=1-\beta$ and $t_\beta=\lceil d_\beta^{-1}\rceil$.
When $t\le t_\beta$, $td_\beta\le2$ and
$1-\beta^t\ge1-e^{-td_\beta}\ge c\,td_\beta$ for a universal $c>0$.
Thus the square root in \eqref{eq:weight-squares} is at most a universal
constant times $t^{-1/2}$.
When $t>t_\beta$, it is at most a universal constant times
$\sqrt{d_\beta}$.
Use $\sum_{t=1}^m t^{-1/2}\le2\sqrt m$ and
$t_\beta\le2/d_\beta$ to average these two regimes and obtain
\eqref{eq:gamma-bound}.
\end{proof}

\paragraph{An optional sharper variance term.}
Because the last fluctuation in \eqref{eq:tracking-decomposition} contains
no $Z_t$, the stochastic term in \eqref{eq:tracking-main} may be replaced by
\begin{equation}
 \frac{\sqrt R}{T}\sum_{t=1}^T
 \left\{\frac{(1-\pi)C^2}{B}\sum_{j=1}^t w_{t,j}^2
                  +D\nu^2\sum_{j=1}^{t-1}w_{t,j}^2\right\}^{1/2}.
 \label{eq:sharper-pre-variance}
\end{equation}
The displayed main theorem uses the larger $\sqrt R v_C\Gamma_T$ for
simplicity. This overbound is not an assumption that the current noise is
part of the conditional clean state.

\subsection{Conditional alignment and the finite-horizon theorem}
Let
\begin{equation}
 \delta_t(\alpha)=\frac1{s_t}\sum_\ell
    \pos{\alpha\norm{A_t^{(\ell)}}_*
           -\inner{A_t^{(\ell)}}{Q_\ell(A_t^{(\ell)})}_F}.
 \label{eq:delta-quality}
\end{equation}
This is $\G_t$-measurable and its expected average is
$\qual_T(\alpha)$. By its definition and the unit norm of every $Q_\ell$,
\begin{align}
 \inner{\bar A_t}{Q(A_t)}
 &\ge\alpha\norm{\bar A_t}_{\psum}-\delta_t(\alpha),\\
 \inner{g_t}{Q(A_t)}
 &\ge\alpha\norm{g_t}_{\psum}
             -(\alpha+1)\norm{\bar A_t-g_t}_{\psum}
             -\delta_t(\alpha).
 \label{eq:alignment-proof}
\end{align}
Here $Q(A_t)$ means the collection of separate matrix maps, not a map
mixing their coordinates.

\begin{proof}[Proof of Theorem~\ref{thm:stationarity}]
The difference between the conditional mean update and $Q(A_t)$ has
maximum operator norm at most $b$. Since $g_t$ is $\G_t$-measurable,
\eqref{eq:alignment-proof} and duality imply
\begin{equation}
 \E[\inner{g_t}{O_t}\mid\G_t]
 \ge(\alpha-b)\norm{g_t}_{\psum}
   -(\alpha+1)\norm{\bar A_t-g_t}_{\psum}-\delta_t(\alpha).
 \label{eq:conditional-descent-alignment}
\end{equation}
Apply \eqref{eq:product-descent} to $H=-\eta O_t$, take total
expectations, and sum over $t$. The objective terms telescope, and
$f(W_T)\ge f_*$ gives
\begin{align*}
 (\alpha-b)\frac1T\sum_t\E\norm{g_t}_{\psum}
 &\le\frac{D_0}{\eta T}
 +(\alpha+1)\frac1T\sum_t\E\norm{\bar A_t-g_t}_{\psum}
 +\qual_T(\alpha)+\frac{L\eta}{2T}\sum_t\E\norm{O_t}_{\pmax}^2.
\end{align*}
Use Lemma~\ref{lem:tracking} and the assumed second-moment bound.
Division by the positive deterministic $\alpha-b$ yields the theorem.
Its two algorithmic specializations follow from
Theorem~\ref{thm:adaptive-bias} and Proposition~\ref{prop:moments}.
\end{proof}

\begin{proof}[Proof of Corollary~\ref{cor:comparison}]
Use common deterministic clipping and quality envelopes as specified,
and abbreviate $N_1=\num_T(1;\alpha)$, $N_K=\num_T(K;\alpha)$.
All denominator factors are positive. The comparison
\[
 \frac{N_K}{\alpha-b_{\rm bc}}
 <\frac{N_1}{\alpha-b_{\rm mu}}
\]
is equivalent, by multiplication of these denominators, to
\[
 N_K-N_1<\frac{b_{\rm mu}-b_{\rm bc}}{\alpha-b_{\rm mu}}N_1.
\]
Direct subtraction of \eqref{eq:certificate-numerator} gives
\[
 N_K-N_1=L\eta(K-1)
       \left[(\alpha+1)\frac{\beta}{1-\beta}+\frac{K+1}{2}\right].
\]
This proves the necessary-and-sufficient comparison \emph{of those two
numbers}. It gives no lower bound on either algorithm's actual stationarity.
\end{proof}

\section{Ordinary DP-Muon without a heat-bias denominator}
\label{app:ordinary}
For ordinary Smooth-NS DP-Muon, one may work with the noisy normalized
momentum $\bar M_t=M_t/s_t$ directly. Define
\begin{equation}
 \qual_T^M(\alpha)=\frac1T\sum_t\E\left[
   \frac1{s_t}\sum_\ell
     \pos{\alpha\norm{M_t^{(\ell)}}_*
                 -\inner{M_t^{(\ell)}}{Q_\ell(M_t^{(\ell)})}_F}\right].
 \label{eq:noisy-quality}
\end{equation}
\begin{theorem}[Noisy-momentum stationarity bound]
\label{thm:ordinary-direct}
Under the matrix-only setup of Theorem~\ref{thm:stationarity}, ordinary
DP-Muon satisfies, for each deterministic $\alpha\in(0,1]$,
\begin{equation}
 \frac1T\sum_t\E\norm{g_t}_{\psum}
 \le\frac1\alpha\left\{\frac{D_0}{\eta T}+\frac{L\eta}{2}
 +(\alpha+1)\left[\bias_T(C)+\sqrt R\,v_C\Gamma_T(\beta)
                         +\frac{L\eta\beta}{1-\beta}\right]
 +\qual_T^M(\alpha)\right\}.
 \label{eq:ordinary-direct}
\end{equation}
At $\alpha=1$,
$\qual_T^M(1)\le R(\gamma_{\max}+C+\nu\sqrt D)/\sqrt{2e(\kappa+1)^q}$.
No assumption $b_{\rm mu}<1$ is needed.
\end{theorem}
\begin{proof}
Expand $\bar M_t=\sum_{j=1}^t w_{t,j}(g_j+b_j+\xi_j)$.
The proof of Lemma~\ref{lem:tracking}, now retaining current noise,
bounds its average tracking error by the same expression with $K=1$.
The alignment argument \eqref{eq:alignment-proof} is now applied pathwise
to $Q(M_t)$ and $\bar M_t$, with the realized defect from
\eqref{eq:noisy-quality}. There is no conditional output-bias term.
Use \eqref{eq:product-descent} and $\norm{Q(M_t)}_{\pmax}\le1$,
telescope, and divide by $\alpha$.
The gap-free bound was established at the end of
Appendix~\ref{app:ns}.
\end{proof}

\subsection{A fixed-noise light-tail rate with an explicit NS schedule}
\begin{corollary}[An optimization asymptotic, not fixed-total privacy]
\label{cor:light-tail-rate}
For a family of matrix-only runs indexed by $T\ge2$, assume exact clipping
$\tau_c=0$ and that $N,B,\pi,D,R,L,D_0,\sigma,\gamma_{\max}$ and
$\kappa\ge1$ are independent of $T$.
Suppose some $K_0>0$ uniformly satisfies
\begin{equation}
 \sup_{T\ge2}\sup_{t\le T}
 \E\left[\frac1N\sum_i
          \exp\!\left(\norm{G_{t,i}}_2^2/K_0^2\right)\right]\le e.
 \label{eq:uniform-light-tail}
\end{equation}
Choose
\begin{equation}
 \eta_T=\eta_0T^{-3/4},\quad \beta_T=1-T^{-1/2},\quad
 C_T=K_0\sqrt{2\log T},\quad
 q_T=\max\left\{1,\left\lceil\frac{\log T}{2\log(\kappa+1)}\right\rceil\right\}.
 \label{eq:light-tail-schedule}
\end{equation}
Then ordinary Smooth-NS DP-Muon satisfies
\begin{equation}
 \frac1T\sum_{t=1}^T\E\norm{g_t}_{\psum}
                 =O(T^{-1/4}\sqrt{\log T}).
 \label{eq:light-tail-rate}
\end{equation}
This assertion holds with fixed \emph{noise multiplier} $\sigma$;
it does not assert a fixed total $(\eps,\delta)$ budget or a rate for a
fixed five-step orthogonalizer.
\end{corollary}
\begin{proof}
Averaged Markov and tail integration in \eqref{eq:uniform-light-tail} give
\[
 \bias_T(C)\le\frac{e\sqrt R K_0^2}{2C}e^{-C^2/K_0^2}.
\]
Thus $\bias_T(C_T)=O(T^{-2}/\sqrt{\log T})$.
The multiplier is fixed but the coordinate noise
$\nu_T=\sigma C_T/B$ is $O(\sqrt{\log T})$.
Consequently $v_{C_T}=O(\sqrt{\log T})$.
Equation~\eqref{eq:gamma-bound} gives
$\Gamma_T(\beta_T)=O(T^{-1/4})$.
The optimization and momentum-drift terms are respectively
$D_0/(\eta_TT)=O(T^{-1/4})$ and
$\eta_T\beta_T/(1-\beta_T)=O(T^{-1/4})$.
Finally, $(\kappa+1)^{q_T}\ge T^{1/2}$, so the gap-free noisy-momentum
quality defect is $O(T^{-1/4}\sqrt{\log T})$.
Substitute these bounds into \eqref{eq:ordinary-direct} with $\alpha=1$.
\end{proof}

\begin{corollary}[A uniform second-moment version]
\label{cor:second-moment-rate}
Under the same fixed-parameter conditions, replace
\eqref{eq:uniform-light-tail} by
$\sup_{T,t\le T}\E N^{-1}\sum_i\norm{G_{t,i}}_2^2\le G_2^2$.
Take
\[
 \eta_T=\eta_0T^{-5/6},\quad \beta_T=1-T^{-2/3},\quad C_T=C_0T^{1/6},
 \quad q_T=\max\left\{1,
        \left\lceil\frac{2\log T}{3\log(\kappa+1)}\right\rceil\right\}.
\]
Then $T^{-1}\sum_t\E\norm{g_t}_{\psum}=O(T^{-1/6})$.
\end{corollary}
\begin{proof}
Use $(x-C)_+\le x^2/C$ to obtain
$\bias_T(C_T)=O(T^{-1/6})$.
Here $v_{C_T}=O(T^{1/6})$,
$\Gamma_T(\beta_T)=O(T^{-1/3})$, and
$\eta_T/(1-\beta_T)=O(T^{-1/6})$.
The initial objective term is also $O(T^{-1/6})$.
Since $(\kappa+1)^{q_T}\ge T^{2/3}$, the quality defect is at most
$O(T^{1/6}T^{-1/3})=O(T^{-1/6})$.
Apply Theorem~\ref{thm:ordinary-direct} at $\alpha=1$.
\end{proof}
The derivative envelopes of a growing-$q_T$ map need not remain bounded.
These two corollaries therefore do not imply analogous high-order bias
rates for DP-MuonBC without additional derivative control. They also do
not convert into fixed-privacy rates by leaving $\sigma$ unchanged.

\section{Coupled auxiliary coordinates: an explicit hybrid inequality}
\label{app:hybrid}
This appendix distinguishes the privacy guarantee for the complete hybrid
optimizer from its optimization guarantee. Let $\theta=(W,u)$,
$h_t=\nabla_u f(\theta_{t-1})$, and use
\begin{equation}
 \norm{(H,v)}_\times=\max\{\norm{H}_{\pmax},\norm{v}_2\},\qquad
 \norm{(G,h)}_{\times,*}=\norm{G}_{\psum}+\norm{h}_2.
 \label{eq:hybrid-norms}
\end{equation}
Assume full product-space smoothness
$\norm{\nabla f(x)-\nabla f(y)}_{\times,*}\le L\norm{x-y}_\times$
and a lower bound $f_*$. Write the updates as
\begin{equation}
 W_t=W_{t-1}-\eta O_t,\qquad u_t=u_{t-1}-\eta V_t.
 \label{eq:hybrid-updates}
\end{equation}
Any auxiliary-to-matrix learning-rate ratio is included in $V_t$.
Let deterministic constants $K_{\times,1},K_{\times,2}$ satisfy
\begin{equation}
 \sup_t\E\norm{(O_t,V_t)}_\times\le K_{\times,1},\qquad
 \sup_t\E\norm{(O_t,V_t)}_\times^2\le K_{\times,2}^2,
 \label{eq:hybrid-moments}
\end{equation}
and define the nonnegative adverse-descent term
\begin{equation}
 \mathcal U_T=\frac1T\sum_t\E\pos{-\inner{h_t}{V_t}}.
 \label{eq:auxiliary-residual}
\end{equation}
The quantities $\bias_T(C)$ and $\qual_T(\alpha)$ below are evaluated
along this actual hybrid trajectory, using its global clipping factors.
The dimension $D$ and rank sum $R$ still refer to Muon matrices.

\begin{proposition}[Matrix stationarity on a hybrid path]
\label{prop:hybrid-bound}
Suppose the matrix directions satisfy the conditional bias part of
\eqref{eq:generic-update-conditions}, with $b<\alpha\le1$.
Then
\begin{align}
 \frac1T\sum_t\E\norm{\nabla_W f(\theta_{t-1})}_{\psum}
 \le\frac1{\alpha-b}\Bigg\{&\frac{f(\theta_0)-f_*}{\eta T}
 +(\alpha+1)\left[\bias_T(C)+\sqrt R\,v_C\Gamma_T(\beta)
                      +\frac{L\eta K_{\times,1}\beta}{1-\beta}\right]
 \nonumber\\[-1mm]
 &+\qual_T(\alpha)+\frac{L\eta K_{\times,2}^2}{2}
                         +\mathcal U_T\Bigg\}.
 \label{eq:hybrid-bound}
\end{align}
\end{proposition}
\begin{proof}
The Poisson oracle decomposition for the matrix components remains valid
because the current full parameters are fixed given $\F_{t-1}$.
The only change in momentum tracking is the drift:
\[
 \E\norm{g_j-g_t}_{\psum}
 \le L\E\norm{\theta_{j-1}-\theta_{t-1}}_\times
 \le L\eta K_{\times,1}(t-j).
\]
The conditional matrix alignment inequality is unchanged. Full smoothness
gives
\[
 f(\theta_t)\le f(\theta_{t-1})-\eta\inner{g_t}{O_t}
      -\eta\inner{h_t}{V_t}
      +\frac{L\eta^2}{2}\norm{(O_t,V_t)}_\times^2.
\]
Bound $-\inner{h_t}{V_t}\le\pos{-\inner{h_t}{V_t}}$,
take expectations, telescope, and use the full-path tracking bound.
This yields \eqref{eq:hybrid-bound}.
\end{proof}
This inequality makes no assertion that $\mathcal U_T$ is small or
vanishes. A full-gradient stationarity theorem would additionally require
an auxiliary descent condition, for example
$\E\inner{h_t}{V_t}\ge a_u\E\norm{h_t}_2-e_t$ with $a_u>0$
and controlled residuals $e_t$. Repeating the same proof would then control
the sum of matrix and auxiliary gradient norms with denominator
$\min\{\alpha-b,a_u\}$. Such an auxiliary alignment condition is not
proved here for ordinary Adam.

\subsection{A magnitude bound for ordinary auxiliary Adam}
The following elementary bound quantifies movement without presuming
alignment. For one coordinate let $x_j$ be the private gradient sequence,
with no restrictions on its signs or dependence, and set
\[
 m_t=(1-\beta_1)\sum_{j=1}^t\beta_1^{t-j}x_j,\quad
 v_t=(1-\beta_2)\sum_{j=1}^t\beta_2^{t-j}x_j^2,\quad
 \widehat m_t=\frac{m_t}{1-\beta_1^t},\quad
 \widehat v_t=\frac{v_t}{1-\beta_2^t}.
\]
These are the usual bias-corrected moments \citep{kingma2015adam}.

\begin{lemma}[Uniform ordinary-Adam step bound]
\label{lem:adam-magnitude}
Assume $0\le\beta_1<1$, $0<\beta_2<1$, $\beta_1^2<\beta_2$, and
$\eps_A>0$. For every $t$ and every real sequence $(x_j)$,
\begin{equation}
 \left|\frac{\widehat m_t}{\sqrt{\widehat v_t}+\eps_A}\right|
 \le\mathsf A_*
 :=\frac1{\sqrt{(1-\beta_2)(1-\beta_1^2/\beta_2)}}.
 \label{eq:adam-bound}
\end{equation}
Consequently an auxiliary Adam block of dimension $d_u$ and learning-rate
ratio $r_\eta\ge0$ has $\norm{V_t}_2\le r_\eta\sqrt{d_u}\mathsf A_*$.
\end{lemma}
\begin{proof}
Weighted Cauchy--Schwarz yields
\[
 \left(\sum_{k=0}^{t-1}\beta_1^k x_{t-k}\right)^2
 \le\left(\sum_{k=0}^{t-1}(\beta_1^2/\beta_2)^k\right)
          \left(\sum_{k=0}^{t-1}\beta_2^k x_{t-k}^2\right).
\]
When $\widehat v_t>0$, divide by it and include startup normalization to get
\[
 \frac{\widehat m_t^2}{\widehat v_t}
 \le\frac{(1-\beta_1)^2}{1-\beta_2}
       \frac{1-\beta_2^t}{(1-\beta_1^t)^2}
       \sum_{k=0}^{t-1}(\beta_1^2/\beta_2)^k
 \le\frac1{(1-\beta_2)(1-\beta_1^2/\beta_2)}.
\]
The last step uses $1-\beta_1^t\ge1-\beta_1$.
A positive denominator stabilizer only decreases the ratio.
If $\widehat v_t=0$, all $x_j$ are zero, so the update is zero.
The Euclidean bound follows by summing the squared coordinate bounds.
\end{proof}
For Smooth-NS matrix directions, one may take
$K_{\times,1}=K_{\times,2}=K_\lambda+r_\eta\sqrt{d_u}\mathsf A_*$,
by the triangle inequality in $L^2$. These constants may be loose.
They do not establish auxiliary descent, do not justify omitting
$\mathcal U_T$, and do not apply unchanged to a variance-subtracted AdamBC
denominator. Matrix shape rescaling and other auxiliary rules must be
included in their respective movement bounds.

\section{A sufficient fixed-total-privacy Gaussian construction}
\label{app:fixed-privacy}
The accountant substitution \eqref{eq:privacy-utility} is the primary
fixed-budget statement. The following non-subsampled calculation illustrates
why a shrinking fresh-noise scale requires an explicitly specified schedule.
It is a sufficient construction, not a lower bound for all DP mechanisms.

Consider sensitivity $\Delta_T=C_T/B_T$ and per-coordinate Gaussian noise
$\nu_T=\sigma_T\Delta_T$. Two equal-covariance Gaussian laws with means
separated by at most $\Delta_T$ have order-$\zeta$ R\'enyi divergence at
most $\zeta\Delta_T^2/(2\nu_T^2)=\zeta/(2\sigma_T^2)$.
This follows by completing the square in the defining Gaussian integral,
and is the standard Gaussian RDP calculation \citep{mironov2017renyi}.
After $T$ adaptive steps the RDP bound is $\zeta\varrho_T$, where
\begin{equation}
 \varrho_T=\frac{T}{2\sigma_T^2}.
 \label{eq:gaussian-rho}
\end{equation}
Minimizing the basic RDP conversion over $\zeta>1$ gives the valid
privacy parameter
\begin{equation}
 \eps_{\rm cert}=\varrho_T+2\sqrt{\varrho_T\log(1/\delta)}.
 \label{eq:gaussian-dp}
\end{equation}
For a fixed positive target $\eps$ and $0<\delta<1$, let
\begin{equation}
 \bar\varrho_{\eps,\delta}
 =\left(\sqrt{\log(1/\delta)+\eps}-\sqrt{\log(1/\delta)}\right)^2.
 \label{eq:rho-target}
\end{equation}
Thus the public choice
\begin{equation}
 \sigma_T=\sqrt{\frac{T}{2\bar\varrho_{\eps,\delta}}}
 \label{eq:sufficient-sigma}
\end{equation}
certifies the target and yields
\begin{equation}
 \nu_T=\frac{C_T\sqrt T}{B_T\sqrt{2\bar\varrho_{\eps,\delta}}}.
 \label{eq:sufficient-noise}
\end{equation}
For this specified calibration, taking
$B_T/C_T=\Theta(T^{a+1/2})$ gives $\nu_T=\Theta(T^{-a})$.
Under additional uniform derivative envelopes, the matched order-$K$
conditional-bias bound is then $O(T^{-2a(K+1)})$.
This is a statement about a conditional bias term; clipping, sampling,
update moments and map quality still enter the optimization bound.

An RDP upper-bound certificate is sufficient and can be conservative.
Equations~\eqref{eq:sufficient-sigma}--\eqref{eq:sufficient-noise} therefore
do not establish that \emph{every} $(\eps,\delta)$-DP optimizer must obey
the same noise lower bound or sample-size scaling.
With Poisson subsampling, use the corresponding valid
$\rho_{\rm SGM}$ rather than the non-subsampled formula.
A fixed dataset, fixed lot size and fixed clipping threshold do not yield
$\nu_T\to0$ under \eqref{eq:sufficient-sigma}; this is why the fixed-noise
optimization corollaries in Appendix~\ref{app:ordinary} are not fixed-total-
privacy convergence claims.

\section{Structure preservation under privacy noise and direct output perturbation}
\label{app:structure}
\label{sec:structure}

In exact arithmetic, $Q(A)$ changes scalar singular responses in an SVD of
$A$ while retaining the associated singular directions. For Smooth-NS,
the scalar response is nonnegative and monotone on its normalized
singular-value range, so the ordering of the singular directions is
preserved as well. The practical Jordan response need not be monotone.
A privacy perturbation occurs \emph{before} this transformation and can
rotate the singular directions of the map input. Thus, preserving the
singular vectors of a noisy input is distinct from preserving informative
subspaces of an underlying signal.

We first give a one-step sufficient condition under which a leading
singular subspace of an adaptive pre-noise state remains stable under the
current privacy perturbation. The result is conditional on the realized
adaptive history and therefore does not require freezing the past.

\subsection{A conditional singular-subspace statement}

\begin{proposition}[Perturbation of leading singular subspaces]
\label{prop:subspaces}
Let $A,Z\in\R^{m\times n}$ and $\widetilde A=A+Z$.
For $1\le k<\min(m,n)$, suppose
\[
g_k=\sigma_k(A)^2-\sigma_{k+1}(A)^2>0
\]
and define
\[
e_Z
=
2\norm{A}_\op\norm{Z}_\op+\norm{Z}_\op^2.
\]
If $e_Z<g_k/2$, and $U_k,\widetilde U_k$ are orthonormal bases
of the leading $k$ left singular subspaces of $A$ and $\widetilde A$,
respectively, then
\begin{equation}
\norm{(I-U_kU_k^\top)\widetilde U_k}_\op
=
\norm{\sin\Theta(\widetilde U_k,U_k)}_\op
\le
\frac{e_Z}{g_k-e_Z}.
\label{eq:subspace-bound}
\end{equation}
The same bound holds for the corresponding right singular subspaces.
\end{proposition}

\begin{proof}
Set
\[
H=AA^\top,
\qquad
\widetilde H
=
\widetilde A\widetilde A^\top
=
H+E.
\]
Expansion gives
\[
E=AZ^\top+ZA^\top+ZZ^\top,
\]
and therefore
\[
\norm{E}_\op
\le
2\norm{A}_\op\norm{Z}_\op+\norm{Z}_\op^2
=
e_Z.
\]
The variational characterization of symmetric eigenvalues yields
\[
|\lambda_i(\widetilde H)-\lambda_i(H)|
\le
\norm{E}_\op.
\]
Since $e_Z<g_k/2$, the leading $k$ eigenspace of $\widetilde H$
remains separated from the remaining eigenspace.

Let $U_\perp$ complete $U_k$ to an orthonormal eigenbasis of $H$,
let
\[
H_\perp=U_\perp^\top H U_\perp,
\]
and let $\widetilde\Lambda_k$ be the diagonal matrix containing the
leading $k$ eigenvalues of $\widetilde H$.
Choose $\widetilde U_k$ to consist of corresponding eigenvectors.
Then
\[
X:=U_\perp^\top\widetilde U_k
\]
satisfies the Sylvester equation
\[
X\widetilde\Lambda_k-H_\perp X
=
U_\perp^\top E\widetilde U_k
=:F.
\]
The smallest eigenvalue of $\widetilde\Lambda_k$ exceeds the largest
eigenvalue of $H_\perp$ by at least
\[
g_k-e_Z>0.
\]
Hence the separated Sylvester equation admits the integral representation
\[
X
=
\int_0^\infty
e^{sH_\perp}
F
e^{-s\widetilde\Lambda_k}\,ds,
\]
which implies
\[
\norm{X}_\op
\le
\frac{\norm{F}_\op}{g_k-e_Z}
\le
\frac{e_Z}{g_k-e_Z}.
\]
Finally,
\[
\norm{X}_\op
=
\norm{(I-U_kU_k^\top)\widetilde U_k}_\op,
\]
which proves \eqref{eq:subspace-bound}.
Applying the same argument to $A^\top A$ and
$\widetilde A^\top\widetilde A$ proves the right-subspace statement.
\end{proof}

For completeness, a conservative conditional Gaussian norm estimate is
available without any distributional assumption on the adaptive state.
Let
\[
Z\sim\mathcal N(0,\nu^2 I_{mn}),
\qquad
0<\omega<1,
\]
and define
\begin{equation}
z_\omega
=
2\nu
\sqrt{
2\big[(m+n)\log 9+\log(2/\omega)\big]
}.
\label{eq:gaussian-net-bound}
\end{equation}
Then
\[
\Pr\{\norm{Z}_\op>z_\omega\}
\le
\omega.
\]
Indeed, choose $1/4$-nets of the two unit spheres with cardinalities at
most $9^m$ and $9^n$, respectively. Approximating a maximizing pair of
unit vectors gives
\[
\norm{Z}_\op
\le
2\max_{u,v\text{ in the nets}}|u^\top Zv|.
\]
Each bilinear form is a centered Gaussian with variance $\nu^2$, and a
Gaussian tail bound followed by a union bound gives
\eqref{eq:gaussian-net-bound}.

Conditionally on $\G_t$, the same estimate applies to the fresh privacy
noise $Z_t$. Consequently, replacing $\norm{Z_t}_\op$ in
Proposition~\ref{prop:subspaces} by the high-probability envelope
$z_\omega$ gives a one-step conditional subspace-stability guarantee with
probability at least $1-\omega$.

The reference state $A_t$ already contains the effects of earlier noisy
updates. Proposition~\ref{prop:subspaces} therefore controls the
\emph{fresh} perturbation of the current adaptive state, rather than its
distance from an entirely nonprivate training trajectory. It identifies
a regime in which informative singular subspaces can survive the current
privacy release, and also the complementary limitation: subspaces whose
spectral gaps are comparable to the perturbation scale need not remain
stable.

For Smooth-NS, the scalar response is nonnegative and monotone by
Lemma~\ref{lem:scalar-residual}, so the map preserves the ordered singular
directions of its own input. The practical Jordan response need not be
monotone, and therefore need not preserve their ordering. Moreover, the
inverse-heat DP-MuonBC direction is a linear combination of outputs
evaluated at different probed matrices and need not preserve any single
realization's singular vectors.

DP-MuonBC addresses a different phenomenon from singular-subspace
perturbation: it reduces the systematic mean distortion created by
applying a nonlinear map to Gaussian-perturbed momentum. It does not
remove the underlying random perturbation and cannot recover information
that has already been obscured in a low-signal-to-noise regime.
Accordingly, the result above should not be interpreted as establishing
that all nonprivate Hessian- or gradient-structure advantages associated
with Muon \citep{shen2025} persist at arbitrary privacy levels.

\subsection{Sensitivity of direct direction perturbation}
\label{app:output-noise}

One alternative is to compute a Muon direction first and add centered
noise directly to that direction. If no nonlinear transformation is
applied afterward, this avoids the particular heat-smoothing bias studied
for gradient perturbation. The relevant privacy sensitivity, however, is
then the sensitivity of the \emph{direction} rather than that of the
clipped gradient.

Let
\[
L_Q
=
\sup_X\norm{DQ(X)}_{F\leftarrow F},
\]
which is finite for the smooth map considered in this paper.
Consider one current query with its previous \emph{private} state held
fixed. If two neighboring clipped averages satisfy
\[
\norm{\bar G-\bar G'}_F\le C/B,
\]
then
\begin{equation}
\norm{
Q(\beta M_{t-1}+\bar G)
-
Q(\beta M_{t-1}+\bar G')
}_F
\le
\min\{L_Q C/B,\,2\sqrt r\}.
\label{eq:output-sensitivity}
\end{equation}
The Lipschitz bound follows by integrating $DQ$ along the line segment
between the two inputs. The second bound follows from
$\norm{Q(X)}_\op\le1$ and therefore
$\norm{Q(X)}_F\le\sqrt r$.

Lemma~\ref{lem:regularity} gives
\[
L_Q=O(\gamma^{-1})
\]
when the remaining map parameters are fixed. Thus direct direction
perturbation is possible for the smoothed map, but its sensitivity is not
automatically the same $C/B$ sensitivity as the clipped-gradient release.

The ideal unregularized polar map behaves differently. In one dimension,
\[
\Polar(x)=\operatorname{sign}(x),
\qquad x\neq0.
\]
The inputs $a$ and $-a$ are only $2a$ apart, whereas their outputs are
distance $2$ apart for every $a>0$. Hence the polar map has no uniform
global Lipschitz constant near degeneracy. This obstruction should not be
transferred to the positively regularized smooth map, whose derivative is
globally bounded.

There is also a statefulness issue. If a direct-output mechanism maintains
an \emph{unprivatized} momentum state, neighboring datasets can change
that state through the entire preceding history. In that setting one
cannot simply fix $M_{t-1}$ when computing sensitivity; the sensitivity
of the full stateful mechanism must instead be controlled.

Centered Gaussian noise added directly to a correctly calibrated direction,
with no subsequent nonlinear map, is conditionally unbiased for that
direction. The realized noisy direction is nevertheless unbounded.
Projecting, normalizing, or re-orthogonalizing it afterward introduces
another nonlinear transformation and can again create output distortion.

Finally, isotropic Frobenius-Gaussian perturbation is a convenient
privacy mechanism rather than an assertion that it is optimal for matrix
structure. For example,
\[
\norm{A}_F\le\sqrt r\,\norm{A}_\op
\]
shows that imposing only an operator-norm sensitivity bound and then
converting it to isotropic Gaussian noise does not automatically reduce
the privacy cost. More matrix-aware mechanisms could exploit a public
subspace, public covariance structure, or an appropriately private
structural estimate, but they require their own sensitivity and privacy
analysis and are outside the global-clipping mechanism studied here.

\section{Additional Experimental Details}
\label{app:experiments-details}

\subsection{Experimental protocol}

We fine-tune GPT-2 on E2E for 10 epochs with maximum sequence length 100.
The packaged data contain 42,061 training references, 4,672 validation
references, and 4,693 test references; generation metrics are evaluated
on 547 validation and 630 test prompts.
The physical batch size is 16 and gradient accumulation is 64, giving an
expected lot size of 1,024.
Sampling is Poisson, and privacy is accounted using RDP with
$\delta=8\times10^{-6}$.

All methods use a single global clipping operation.
For each example, the gradients of all trainable parameters are viewed as
one vector, clipped once in Euclidean/Frobenius norm, and included in one
joint Gaussian release.
Thus each optimizer step contributes one private Gaussian mechanism;
all subsequent optimizer operations are post-processing.

For Muon-based methods, hidden matrix parameters use matrix momentum
$0.95$ and five matrix-map iterations.
Vectors, token and positional embeddings, and the language-model head are
excluded from Muon and use auxiliary Adam with learning rate $0.002$.
DP-Muon uses the practical five-step Jordan Newton--Schulz map without
bias correction.

The DP-MuonBC configuration reported in the main text uses first-order
Smooth-NS inverse-heat correction with strength
$\lambda=1.25$, one antithetic probe, $\kappa=2$, and
$\gamma_{\mathrm{ns}}=10^{-4}$.
The strength $\lambda=1.25$ is a validation-selected extrapolation
parameter.  The exact cancellation of the leading second-order heat term
occurs at $\lambda=1$; the general arbitrary-strength analysis applies to
the experimentally selected $\lambda=1.25$ configuration.

\subsection{Learning-rate selection}

Table~\ref{tab:app-lr} gives the learning rates selected by validation
NLL for the configurations reported in Table~\ref{tab:main-results}.
Except for DP-SGD, the selected learning rates are bracketed by worse
validation-NLL configurations.
For DP-SGD, performance continued to improve up to the largest tested
learning rate at each privacy budget, so these runs are reported as the
best observed DP-SGD configurations rather than as bracketed optima.

\begin{table}[htbp]
\centering
\small
\begin{tabular}{lrrrr}
\toprule
Method
& $\varepsilon\approx1$
& $\varepsilon\approx2$
& $\varepsilon\approx4$
& $\varepsilon\approx8$\\
\midrule
DP-SGD$^\dagger$
    & 0.0880 & 0.0520 & 0.1040 & 0.1920\\
DP-Adam
    & 0.0010 & 0.0010 & 0.0015 & 0.0015\\
DP-AdamBC
    & 0.0005 & 0.0010 & 0.0015 & 0.0020\\
DP-Muon
    & 0.0050 & 0.0060 & 0.0060 & 0.0060\\
DP-MuonBC
    & 0.0070 & 0.0070 & 0.0080 & 0.0080\\
\bottomrule
\end{tabular}
\caption{
Validation-selected learning rates for the main E2E comparison.
For Muon-based methods, the listed value is the matrix learning rate.
$^\dagger$The best measured DP-SGD point is at the upper search boundary.
}
\label{tab:app-lr}
\end{table}

\subsection{Generation metrics}

Tables~\ref{tab:app-bleu} and~\ref{tab:app-rouge} report BLEU and
ROUGE-L for the same validation-NLL-selected seed-0 configurations used
in Table~\ref{tab:main-results}.
These metrics are not used for hyperparameter selection.
Their optimizer ordering does not always coincide with the NLL ordering.

\begin{table}[htbp]
\centering
\small
\begin{tabular}{lrrrr}
\toprule
Method
& $\varepsilon\approx1$
& $\varepsilon\approx2$
& $\varepsilon\approx4$
& $\varepsilon\approx8$\\
\midrule
DP-SGD
    & 32.275 & 30.456 & 32.908 & 41.100\\
DP-Adam
    & 51.458 & 55.893 & 60.661 & 62.607\\
DP-AdamBC
    & 52.307 & 57.810 & 62.159 & 62.915\\
DP-Muon
    & \textbf{56.088} & \textbf{59.707} & 62.106 & \textbf{63.783}\\
DP-MuonBC
    & 55.913 & 59.602 & \textbf{62.591} & 63.664\\
\bottomrule
\end{tabular}
\caption{
E2E test BLEU for the validation-NLL-selected seed-0 configurations;
higher is better.
}
\label{tab:app-bleu}
\end{table}

\begin{table}[htbp]
\centering
\small
\begin{tabular}{lrrrr}
\toprule
Method
& $\varepsilon\approx1$
& $\varepsilon\approx2$
& $\varepsilon\approx4$
& $\varepsilon\approx8$\\
\midrule
DP-SGD
    & 60.224 & 59.703 & 59.683 & 57.438\\
DP-Adam
    & 60.982 & 63.110 & 65.968 & 66.659\\
DP-AdamBC
    & 61.298 & 64.814 & \textbf{66.762} & \textbf{66.976}\\
DP-Muon
    & 63.612 & 64.808 & 66.578 & 66.957\\
DP-MuonBC
    & \textbf{63.917} & \textbf{65.050} & 66.717 & 66.705\\
\bottomrule
\end{tabular}
\caption{
E2E test ROUGE-L for the validation-NLL-selected seed-0 configurations;
higher is better.
}
\label{tab:app-rouge}
\end{table}

\subsection{Bias-correction strength ablation}

We additionally vary the strength of the first-order Smooth-NS correction
at $\varepsilon\approx2$.
The learning rate $0.007$ for the corrected configuration is locally
bracketed by learning rates $0.006$ and $0.008$.
Within the tested strength range, validation NLL improves gradually as
$\lambda$ increases from $0.5$ to $1.25$.

\begin{table}[htbp]
\centering
\small
\begin{tabular}{lrrr}
\toprule
Configuration & Matrix LR & Val NLL & Test NLL\\
\midrule
DP-Muon
    & 0.006 & 0.534040 & 0.543061\\
Jordan1BCx5
    & 0.006 & 0.533739 & 0.542853\\
Smooth-NS, first order, $\lambda=0.50$
    & 0.007 & 0.533460 & 0.542709\\
Smooth-NS, first order, $\lambda=0.75$
    & 0.007 & 0.533438 & 0.542685\\
Smooth-NS, first order, $\lambda=1.00$
    & 0.007 & 0.533423 & 0.542670\\
Smooth-NS, first order, $\lambda=1.25$
    & 0.007 & \textbf{0.533411} & \textbf{0.542657}\\
\bottomrule
\end{tabular}
\caption{
First-order bias-correction strength ablation at
$\varepsilon=1.995747$ and seed 0.
The selected DP-MuonBC configuration uses $\lambda=1.25$.
}
\label{tab:app-strength}
\end{table}

The best tested strength is therefore $\lambda=1.25$.
The difference between $\lambda=1$ and $\lambda=1.25$ is small:
their test NLLs are $0.542670$ and $0.542657$, respectively.
Thus the empirical optimum above one should be interpreted as a tuned
extrapolation strength rather than as evidence that stronger correction
has a better formal bias order.

\subsection{Correction-order ablation}

We also evaluated the second-order Smooth-Jordan correction in the
systematic global-clipping sweep.
As shown in Table~\ref{tab:app-second-order}, it remains close to DP-Muon
but is slightly worse in test NLL at each privacy budget.
Thus increasing the formal correction order does not necessarily improve
optimization utility, consistent with the additional variance and update
magnitude introduced by higher-order extrapolation.

\begin{table}[htbp]
\centering
\small
\begin{tabular}{lrrrr}
\toprule
Method
& $\varepsilon\approx1$
& $\varepsilon\approx2$
& $\varepsilon\approx4$
& $\varepsilon\approx8$\\
\midrule
DP-Muon
    & 0.567287 & 0.543061 & 0.526304 & 0.513344\\
Smooth-Jordan, second order
    & 0.567414 & 0.543427 & 0.526572 & 0.513459\\
\bottomrule
\end{tabular}
\caption{
E2E test NLL for the systematic second-order Smooth-Jordan comparison.
The higher-order correction is close to, but does not improve upon,
the practical DP-Muon baseline in this sweep.
}
\label{tab:app-second-order}
\end{table}

\subsection{Paired multi-seed comparison}

For $\varepsilon\approx1,4,8$, we additionally run DP-Muon and the
selected first-order DP-MuonBC configuration with seeds $0,1,2$.
Table~\ref{tab:app-multiseed} reports the mean and sample standard
deviation.
DP-MuonBC obtains lower mean validation and test NLL at all three privacy
budgets.

\begin{table}[htbp]
\centering
\small
\begin{tabular}{clcc}
\toprule
$\varepsilon$
& Method
& Val NLL
& Test NLL\\
\midrule
1
& DP-MuonBC
& $0.559190 \pm 0.000990$
& $0.564605 \pm 0.001837$\\
1
& DP-Muon
& $0.559705 \pm 0.000882$
& $0.565070 \pm 0.001956$\\
\midrule
4
& DP-MuonBC
& $0.515643 \pm 0.001074$
& $0.523921 \pm 0.001907$\\
4
& DP-Muon
& $0.516220 \pm 0.001210$
& $0.524436 \pm 0.001903$\\
\midrule
8
& DP-MuonBC
& $0.502334 \pm 0.000931$
& $0.511026 \pm 0.001806$\\
8
& DP-Muon
& $0.503007 \pm 0.001010$
& $0.511677 \pm 0.001708$\\
\bottomrule
\end{tabular}
\caption{
Paired E2E seed-$0/1/2$ comparison between DP-Muon and DP-MuonBC.
Values are mean $\pm$ sample standard deviation.
}
\label{tab:app-multiseed}
\end{table}

Across these three privacy budgets, DP-MuonBC has lower validation NLL
for all nine matched seed--budget pairs.
Its mean test-NLL improvements over DP-Muon are approximately
$4.65\times10^{-4}$, $5.15\times10^{-4}$, and
$6.51\times10^{-4}$ at $\varepsilon\approx1,4,8$, respectively.
For $\varepsilon\approx2$, the selected $\lambda=1.25$ configuration is
currently reported only for seed 0, so it is not included in this
multi-seed table.
\end{document}